%% file: Paper.tex
\documentclass[twoside,11pt]{article}
\usepackage{jmlr2e}

\usepackage{algorithm}
\usepackage{algorithmicx}
\usepackage[noend]{algpseudocode}
\usepackage{amsfonts}
\usepackage{amsmath}
\usepackage{amssymb}
\usepackage{bbm}
\usepackage{bm}
\usepackage{color}
\usepackage{dirtytalk}
\usepackage{dsfont}
\usepackage{enumerate}
\usepackage{graphicx}
\usepackage{listings}
\usepackage{mathtools}
\usepackage{subfigure}
\usepackage{times}
\usepackage{xspace}

\usepackage[usenames,dvipsnames]{xcolor}
\hypersetup{
  pdftex,
  pdffitwindow=true,
  pdfstartview={FitH},
  pdfnewwindow=true,
  colorlinks,
  linktocpage=true,
  linkcolor=Green,
  urlcolor=Green,
  citecolor=Green
}
\usepackage[capitalize,noabbrev]{cleveref}

\usepackage[backgroundcolor=white]{todonotes}

\newcommand{\commentout}[1]{}
\newcommand{\junk}[1]{}

\Crefname{corollary}{Corollary}{Corollaries}
\Crefname{proposition}{Proposition}{Propositions}
\Crefname{theorem}{Theorem}{Theorems}
\Crefname{definition}{Definition}{Definitions}
\Crefname{assumption}{Assumption}{Assumptions}
\Crefname{example}{Example}{Examples}
\Crefname{remark}{Remark}{Remarks}
\Crefname{setting}{Setting}{Settings}
\Crefname{lemma}{Lemma}{Lemmas}

\usepackage{thmtools}
\declaretheorem[name=Theorem,refname={Theorem,Theorems},Refname={Theorem,Theorems}]{theorem}
\declaretheorem[name=Lemma,refname={Lemma,Lemmas},Refname={Lemma,Lemmas},sibling=theorem]{lemma}

\declaretheorem[name=Example,refname={Example,Examples},Refname={Example,Examples}]{example}

\newcommand{\cI}{\mathcal{I}}
\newcommand{\cM}{\mathcal{M}}
\newcommand{\cN}{\mathcal{N}}
\newcommand{\cP}{\mathcal{P}}
\newcommand{\cW}{\mathcal{W}}
\newcommand{\eps}{\varepsilon}

\newcommand{\realset}{\mathbb{R}}

\newcommand{\E}[1]{\mathbb{E} \left[#1\right]}
\newcommand{\condE}[2]{\mathbb{E} \left[#1 \,\middle|\, #2\right]}
\newcommand{\Et}[1]{\mathbb{E}_t \left[#1\right]}
\newcommand{\prob}[1]{\mathbb{P} \left(#1\right)}
\newcommand{\condprob}[2]{\mathbb{P} \left(#1 \,\middle|\, #2\right)}
\newcommand{\probt}[1]{\mathbb{P}_t \left(#1\right)}

\newcommand{\abs}[1]{\left|#1\right|}
\newcommand{\ceils}[1]{\left\lceil#1\right\rceil}

\newcommand*\dif{\mathop{}\!\mathrm{d}}
\newcommand{\floors}[1]{\left\lfloor#1\right\rfloor}
\newcommand{\I}[1]{\mathds{1} \! \left\{#1\right\}}

\newcommand{\norm}[1]{\|#1\|}
\newcommand{\normw}[2]{\|#1\|_{#2}}

\newcommand{\set}[1]{\left\{#1\right\}}

\newcommand{\T}{^\top}

\DeclareMathOperator*{\argmax}{arg\,max\,}

\mathchardef\mhyphen="2D

\newcommand{\cosoftelim}{\ensuremath{\tt CoSoftElim}\xspace}

\newcommand{\expthree}{\ensuremath{\tt Exp3}\xspace}

\newcommand{\gradband}{\ensuremath{\tt GradBand}\xspace}

\newcommand{\lints}{\ensuremath{\tt LinTS}\xspace}
\newcommand{\linucb}{\ensuremath{\tt LinUCB}\xspace}

\newcommand{\softelim}{\ensuremath{\tt SoftElim}\xspace}

\newcommand{\ts}{\ensuremath{\tt TS}\xspace}
\newcommand{\ucb}{\ensuremath{\tt UCB1}\xspace}
\newcommand{\ucbv}{\ensuremath{\tt UCB\mhyphen V}\xspace}

\ShortHeadings{Meta-Learning Bandit Policies by Gradient Ascent}{Kveton, Mladenov, Hsu, Zaheer, Szepesv\'{a}ri, and Boutilier}

\firstpageno{1}

\begin{document}

\title{Meta-Learning Bandit Policies by Gradient Ascent}

\author{\name Branislav Kveton
  \email bkveton@google.com \\
  \addr Google Research
  \AND
  \name Martin Mladenov
  \email mmladenov@google.com \\
  \addr Google Research
  \AND
  \name Chih-Wei Hsu
  \email cwhsu@google.com \\
  \addr Google Research
  \AND
  \name Manzil Zaheer
  \email manzilzaheer@google.com \\
  \addr Google Research
  \AND
  \name Csaba Szepesv\'{a}ri
  \email szepi@google.com \\
  \addr DeepMind / University of Alberta
  \AND
  \name Craig Boutilier
  \email cboutilier@google.com \\
  \addr Google Research
}

\editor{}

\maketitle

\begin{abstract}
Most bandit policies are designed to either minimize regret in any problem instance, making very few assumptions about the underlying environment, or in a Bayesian sense, assuming a prior distribution over environment parameters. The former are often too conservative in practical settings, while the latter require assumptions that are hard to verify in practice. We study bandit problems that fall between these two extremes, where the learning agent has access to \emph{sampled bandit instances} from an unknown prior distribution $\cP$ and aims to achieve high reward on average over the bandit instances drawn from $\cP$. This setting is of a particular importance because it lays foundations for \emph{meta-learning} of bandit policies and reflects more realistic assumptions in many practical domains. We propose the use of parameterized bandit policies that are \emph{differentiable} and can be optimized using \emph{policy gradients}. This provides a broadly applicable framework that is easy to implement. We derive reward gradients that reflect the structure of bandit problems and policies, for both non-contextual and contextual settings, and propose a number of interesting policies that are both differentiable and have low regret. Our algorithmic and theoretical contributions are supported by extensive experiments that show the importance of baseline subtraction, learned biases, and the practicality of our approach on a range problems.
\end{abstract}


\input{Introduction}

\input{Setting}

\input{Optimization}

\input{Gradient}

\input{Algorithms}

\input{AlgorithmsCo}

\input{Experiments}

\input{RelatedWork}

\input{Conclusions}

\bibliography{References}

\clearpage
\onecolumn
\appendix

\input{Concavity}

\clearpage

\input{GradientProofs}

\clearpage

\input{Analysis}

\clearpage

\input{AnalysisCo}

\clearpage

\input{Experiments2}

\clearpage

\input{RNNDetails}

\end{document}

%% file: Introduction.tex
\section{Introduction}
\label{sec:introduction}

A \emph{stochastic bandit} \citep{lai85asymptotically,auer02finitetime,lattimore19bandit} is an online learning problem where a \emph{learning agent} sequentially pulls \emph{arms} with stochastic rewards. The agent aims to maximize its expected cumulative reward over some horizon, but does not know the mean rewards of the arms \emph{a priori}. Hence, it must learn them by pulling the arms. This induces the well-known \emph{exploration-exploitation trade-off}: the agent can \emph{explore} to learn more about arms, or it can \emph{exploit} what it has learned and pull the arm with the highest estimated reward. An example of a bandit is a clinical trial, where an \emph{arm} might correspond to a treatment and its \emph{reward} reflects the random outcome of that treatment for a generic patient. A \emph{contextual bandit} \citep{li10contextual,agrawal13thompson} is a generalization of the bandit where the stochastic reward of an arm depends on observed, but varying, \emph{context}. In the clinical trial example, the context might be the medical history of the treated patient, since an outcome of a treatment is likely to depend on it.

Bandit algorithms are typically evaluated by their \emph{regret}, the difference between cumulative rewards of always pulling the optimal arm and actions of the bandit algorithm. Most commonly, bandit algorithms are designed to have low regret in any problem instance in some fixed problem class \citep{lattimore19bandit}. The regret bounds are either worst-case or instance-dependent. This \emph{classical approach} to algorithm design focuses on improvements on hard problem instances, no matter how unlikely they are, and neglects easy instances, which may be more likely in practice. It may also neglect specific properties of the problem class or objective, if they are hard to capture using formalism.

An alternative to the classical design are \emph{Bayesian bandits} \citep{gittins79bandit,gittins11multiarmed}, where some prior distribution $\cP$ over problem instances is assumed. In Bayesian bandits, the aim is to design \emph{Bayes optimal policies} that perform well in expectation with respect to prior $\cP$, have high \emph{Bayes reward} or equivalently low \emph{Bayes regret}. For specific priors $\cP$, such policies often have relatively simple forms and can be computed using dynamic programming. Unfortunately, it is not known how to design these policies when $\cP$ is unknown or when the context is present, and they are computationally costly to compute in general.

The classical and Bayesian approaches fall on two ends of a spectrum that limit their utility in practice. The classical policies are often conservative and over-explore on easy problem instances, sacrificing strong regret guarantees for practicality. The strong assumptions in Bayesian methods mean that they have limited practical applicability, because realistic problems rarely meet the assumed requirements.

We study an alternative view that falls between these two extremes. Specifically, we focus on bandit problems where the learning agent has access to \emph{sampled bandit instances} from an unknown prior distribution $\cP$. Our aim is to \emph{learn} bandit policies, using these sampled instances, that have high Bayes reward with respect to $\cP$. In essence, we automate learning of Bayesian bandit policies \citep{berry85bandit} without assuming that the prior $\cP$ is known. Our techniques are a form of \emph{meta-learning} \citep{thrun96explanationbased,thrun98lifelong,baxter98theoretical,baxter00model} where we learn good bandit algorithms, specialized to problem instances drawn from $\cP$.

Our approach falls in a useful middle ground of both theoretical and practical importance. In comparison to Bayesian bandits, we make minimal assumptions about the prior $\cP$ and the form of optimized bandit policies; and also do not require that $\cP$ is known. In comparison to the classical designs, we directly exploit domain-specific characteristics encoded in distribution $\cP$, which may be hard to capture in a pure-theory design. Indeed, one prime motivation for our work is that bandit policies are rarely put into practice as analyzed. It is well-known that careful tuning often reduces their regret by an order of magnitude \citep{vermorel05multiarmed,maes12metalearning,kuleshov14algorithms,hsu19empirical}. Unfortunately, practical tuning tends to be \emph{ad hoc} in nature.

We develop a \emph{general, systematic, and data-dependent} approach to learning bandit policies. At a high level, we design bandit policies of specific parametric forms, which are optimized using \emph{policy gradients} \citep{williams92simple,sutton00policy,baxter01infinitehorizon} on sampled problem instances from some underlying distribution $\cP$. This approach is data-dependent and exploits problem structure in ways that classical designs do not, by directly minimizing the quantity of interest, the \emph{actual Bayes regret} with respect to $\cP$. In comparison to Bayesian bandits, we lose guarantees on Bayes optimality, in return for generality and lower computational cost (\cref{sec:comparison existing solutions}).

The variance in policy-gradient optimization can make optimization challenging. To make our approach practical, we incorporate novel forms of \emph{baseline subtraction} to reduce variance. We also propose special \emph{policy classes} that are well suited to our approach, because they are differentiable and have low regret for some choices of the policy parameters. We use such parameter choices as starting points in our optimization, because they guarantee that the initial policies cannot perform too poorly. Two of our proposed policies, \softelim and \cosoftelim, are algorithmically novel. They are randomized yet they do not explore by adding noise to the maximum likelihood estimate, like other randomized exploration schemes in stochastic bandits \citep{agrawal12analysis,agrawal13thompson,russo18tutorial,kveton19garbage,kveton19perturbed2}. We also derive the \emph{reward gradient of Thompson sampling}, which is a result of broad significance given the practical importance of Thompson sampling. Finally, we evaluate our methodology empirically on a range of bandit problems, which highlights its versatility.

The paper is structured as follows. In \cref{sec:setting}, we introduce basic background, notation, and our setting. In \cref{sec:bandit policy optimization}, we propose gradient-based optimization of bandit policies. In \cref{sec:reward gradient}, we derive the reward gradient and suggest baseline subtraction to make the optimization practical. In \cref{sec:mab algorithms,sec:contextual algorithms}, we present our differentiable policies for multi-armed and contextual bandits, respectively; and analyze some of them. In \cref{sec:experiments}, we evaluate our proposed policies and their optimization. We discuss related work in \cref{sec:related work} and conclude in \cref{sec:conclusions}.

Parts of this article appeared in \cite{boutilier20differentiable}. This work extends the meta-learning framework in the earlier paper to contextual bandits and provides additional exposition, theoretical analysis, and experiments.

%% file: Setting.tex
\section{Problem Setting}
\label{sec:setting}

This section introduces basic background, notation, and assumptions used in the remainder of the paper. We introduce stochastic multi-armed bandits in \cref{sec:stochastic multi-armed bandits}, describe Bayesian bandits in \cref{sec:bayesian bandits}, and detail the assumptions underlying our techniques in \cref{sec:framework}.

We adopt the following notation. Let $[n] = \set{1, \dots, n}$. We denote by $x \oplus y$ the concatenation of vectors $x$ and $y$. For any \emph{positive semi-definite (PSD)} matrix $M$, we define $\normw{x}{M} = \sqrt{x\T M x}$. The $d \times d$ identity matrix is denoted $I_d$. We use $\tilde{O}$ for the big-O notation up to logarithmic factors.

\subsection{Stochastic Multi-Armed Bandits}
\label{sec:stochastic multi-armed bandits}

A \emph{stochastic multi-armed bandit} \citep{lai85asymptotically,auer02finitetime,lattimore19bandit} is an online learning problem where a learning agent interacts with the environment by repeatedly making decisions, or pulling \emph{arms}, over a sequence of $n$ interactions, or \emph{rounds}. The agent receives feedback for each arm pull in the form of a real-valued stochastic \emph{reward} and tries to learn the optimal arm. In the standard multi-armed bandit, the reward distributions of arms are fixed. In this work, we also consider \emph{contextual bandits} \citep{li10contextual,chu11contextual,agrawal13thompson}, where the reward distributions depend on some observed \emph{context}, which varies over time. In this case, the choice of the agent at each round can be conditioned on the context at that round. We formalize the contextual setting here, since non-contextual problems can be viewed as a special case with a single fixed context.

Formally, the setting is defined as follows. We assume $K$ arms $i\in [K]$, a horizon of $n$ rounds, and that contexts are vectors in $\realset^d$. In round $t \in [n]$, the agent observes context $x_t \in \realset^d$, pulls arm $I_t \in [K]$, and observes the reward of the pulled arm $I_t$. The \emph{mean reward} of arm $i \in [K]$ in context $x$ is $f_i(x, \theta_*)$, where $f_i: \realset^d \times \Theta \to \realset$ is a known function, $\theta_* \in \Theta$ is a vector of unknown \emph{model parameters}, and $\Theta$ is the set of \emph{feasible model parameters}. The model parameters $\theta_*$ capture all unknown elements of the environment that the learning agent interacts with and we refer to it as a \emph{problem instance}. We denote by $Y_{i, t}$ the \emph{realized reward} of arm $i$ in round $t$ and assume that it is drawn i.i.d.\ from a $\sigma^2$-sub-Gaussian distribution with mean $\condE{Y_{i, t}}{\theta_*, x_t} = f_i(x_t, \theta_*)$. We denote all realized rewards in round $t$ by $Y_t = (Y_{i, t})_{i = 1}^K$ and all $n$-round rewards by $Y = (Y_t)_{t = 1}^n \in \realset^{K \times n}$. The assumption that the rewards of all arms are realized is only to simplify notation. The learning agent only observes the rewards of pulled arms. We assume that the sequence of $n$-round contexts is fixed in advance and denote it by $x_{1 : n} = (x_t)_{t = 1}^n$.

At a high level, all learning agents for bandit problems can be viewed as policies that map the history of the agent to the next pulled arm, or a distribution over next pulled arms. The \emph{history} $H_t$ of the agent at the beginning of round $t$ is a sequence of all past observations of contexts, rewards, and actions in the first $t - 1$ rounds; and the current context in round $t$,
\begin{align*}
  H_t
  = (I_1, \dots, I_{t - 1}, x_1, \dots, x_t, Y_{I_1, 1}, \dots, Y_{I_{t - 1}, t - 1})\,.
\end{align*}
The agent is a \emph{randomized policy} $\pi_w(\cdot \mid H_t) \in \Delta_K$, where $\Delta_K$ is the $K$-dimensional probability simplex and $w \in \cW$ are \emph{policy parameters} that fall within the set of \emph{feasible policy parameters} $\cW$. Thus, $\pi_w(i \mid H_t)$ is the probability that arm $i$ is pulled in round $t$ conditioned on history $H_t$ and the actual pulled arm is drawn as $I_t \sim \pi_w(\cdot \mid H_t)$. To simplify notation, we write $\pi$ instead of $\pi_w$ when the dependence on $w$ is not germane to the discussion. Let $\probt{\cdot} = \condprob{\cdot}{H_t}$ and $\Et{\cdot} = \condE{\cdot}{H_t}$ be the conditional probability and expectation, respectively, of any generic random variable given history $H_t$. We also define $I_{i : j} = (I_\ell)_{\ell = i}^j$ and $I = I_{1 : n}$.

The \emph{expected $n$-round reward} of policy $\pi$ in problem instance $\theta_*$ is
\begin{align}
  r(n, \theta_*; \pi)
  = \condE{\sum_{t = 1}^n Y_{I_t, t}}{\theta_*}\,.
  \label{eq:n-round reward}
\end{align}
Note that the expectation is over both realized rewards $Y$ and pulled arms $I$, while the problem instance $\theta_*$ is fixed. A typical goal for bandit algorithms is to maximize $r(n, \theta_*; \pi)$. This is equivalent to minimizing the \emph{expected $n$-round regret} of $\pi$, defined as
\begin{align}
  R(n, \theta_*; \pi)
  = \condE{\sum_{t = 1}^n Y_{i_{*, t}(\theta_*), t} - Y_{I_t, t}}{\theta_*}\,,
  \label{eq:n-round regret}
\end{align}
where $i_{*, t}(\theta_*) = \argmax_{i \in [K]} f_i(x_t, \theta_*)$ is the \emph{optimal arm} in round $t$ in instance $\theta_*$.

\subsection{Bayesian Bandits}
\label{sec:bayesian bandits}

A \emph{Bayesian bandit} \citep{gittins79bandit,berry85bandit} is a bandit problem where the learning agent interacts with problem instances $\theta_*$ drawn i.i.d.\ from a \emph{prior distribution} $\cP$. Our instances are defined in \cref{sec:stochastic multi-armed bandits}. The learning agent interacts with the environment as follows. First, the instance is sampled as $\theta_* \sim \cP$ and the rewards are realized as $Y \mid \theta_*, x_{1 : n}$. As in \cref{sec:stochastic multi-armed bandits}, we assume that $x_{1 : n}$ is an arbitrary fixed sequence of contexts. Then the agent interacts with instance $\theta_*$ for $n$ rounds, from round $1$. The agent does not know $\theta_*$ and $Y$, but it knows the prior $\cP$.

The quality of Bayesian bandit policies is measured by their $n$-round Bayes reward or regret. The \emph{$n$-round Bayes reward} of policy $\pi$ is
\begin{align}
  r(n; \pi)
  = \E{r(n, \theta_*; \pi)}
  = \E{\sum_{t = 1}^n Y_{I_t, t}}\, ,
  \label{eq:bayes reward}
\end{align}
where expectation is taken over realized rewards, pulled arms, and problem instances $\theta_*$. Note that the expectation over $\theta_*$ is not present in \eqref{eq:n-round reward}. The goal of Bayesian bandit policies is to maximize $r(n; \pi_w)$ with respect to $w$. This is equivalent to minimizing the \emph{$n$-round Bayes regret}, defined as
\begin{align}
  R(n; \pi)
  = \E{R(n, \theta_*; \pi)}
  = \E{\sum_{t = 1}^n Y_{i_{*, t}(\theta_*), t} - Y_{I_t, t}}\,,
  \label{eq:bayes regret}
\end{align}
where $i_{*, t}(\theta_*)$ is defined as in \eqref{eq:n-round regret}.

To further elucidate our setting and notation, we provide examples of Bernoulli and contextual linear bandits below.

\begin{example}[Bernoulli bandit with a uniform beta prior]
\label{ex:bernoulli bandit} The model parameters are a vector of arm means $\theta_* \in [0, 1]^K$, where $(\theta_*)_i$ is the mean reward of arm $i$. Thus we have $f_i(x, \theta_*) = (\theta_*)_i$ and $Y_{i, t} \sim \mathrm{Ber}((\theta_*)_i)$. The problem-instance distribution is $\cP(\theta_*) = \prod_{i = 1}^K \mathrm{Beta}((\theta_*)_i; 1, 1)$.
\end{example}

\begin{example}[Contextual bandit with Gaussian rewards and prior]
\label{ex:contextual bandit} The model parameters are a concatenation of $K$ vectors, $\theta_* = \theta_{1, *} \oplus \dots \oplus \theta_{K, *}$, where vector $\theta_{i, *} \in \realset^d$ corresponds to arm $i$. The mean reward of arm $i$ in context $x$ is $f_i(x, \theta_*) = x\T \theta_{i, *}$ and $Y_{i, t} \sim \cN(x\T \theta_{i, *}, \sigma^2)$ for $\sigma > 0$. The problem-instance distribution is $\cP(\theta_*) = \cN(\theta_*; \mathbf{0}, \sigma_0 I_{K d})$ for some $\sigma_0 > 0$.
\end{example}

Bandit policies can be also described readily using our notation. To illustrate this, we define a non-contextual \emph{$\eps$-greedy policy}, which has a single exploration parameter $\eps$.

\begin{example}[$\eps$-greedy policy]
\label{ex:e-greedy policy} The $\eps$-greedy policy has parameter $\eps \in [0, 1]$ and works as follows. It pulls the best empirical arm with probability $1 - \eps$ and a random arm with probability $\eps$. This can be written in our notation as
\begin{align*}
  \pi_\eps(i \mid H_t)
  = (1 - \eps) \I{i = \argmax_{j \in [K]} \hat{\mu}_{j, t - 1}} + \frac{\eps}{K}\,,
\end{align*}
where $\hat{\mu}_{i, t}$ is the empirical mean of arm $i$ after $t$ rounds, $\hat{\mu}_{i, t} = \sum_{\ell = 1}^t \I{I_\ell = i} Y_{i, \ell}$.
\end{example}

\subsection{Our Framework}
\label{sec:framework}

Bayesian bandits form a  basis for our work. However, unlike in the standard Bayesian model, we do not assume that the prior $\cP$ over bandit instances is known. Specifically, we make no structural assumption about this prior, nor do we assume any prior over its possible parameters. Because of this, we also make no assumption about the structure of the optimal bandit policy. Nevertheless, we impose structure later, for the purpose of learning good policies.

Instead, we assume access to a collection of sampled problem instances $\theta_*$ from an unknown but fixed distribution $\cP$. For each instance, we assume that we can \say{simulate} interactions with it. Each simulation of policy $\pi_w$ on problem instance $\theta_*$ outputs $n$ pulled arms $I \in [K]^n$ by $\pi_w$ and all realized rewards $Y \mid \theta_*, x_{1 : n}$, even for the arms that $\pi_w$ does not pull. When $\theta_*$ is known, the simulation is always possible, by following the stochastic model of interaction in \cref{sec:stochastic multi-armed bandits}. The ability to simulate environments is common in reinforcement learning and is one of the reasons for its recent massive success \citep{silver17mastering}.

The bandit policy $\pi_\eps$ in \cref{ex:e-greedy policy} and our objective in \eqref{eq:bayes reward} showcase the two-level character of our learning problem. At the lower level, $\pi_\eps$ adapts to an unknown problem instance $\theta_*$, since it is a function of the history $H_t$ that depends on $\theta_*$. At the higher level, we optimize \eqref{eq:bayes reward} through $\eps$, so that $\pi_\eps$ adapts to $\theta_*$ as efficiently as possible, on average over $\theta_* \sim \cP$. This is why we refer to our problem as meta-learning of bandit policies.

%% file: Optimization.tex
\section{Bandit Policy Optimization}
\label{sec:bandit policy optimization}

\begin{algorithm}[t]
  \caption{\gradband: Gradient-based optimization of bandit policies.}
  \label{alg:gradband}
  \begin{algorithmic}[1]
    \State \textbf{Inputs:}
    \State \quad Initial policy parameters $w_0 \in \cW$
    \State \quad Number of iterations $L$
    \State \quad Learning rate $\alpha$
    \State \quad Batch size $m$
    \Statex
    \State $w \gets w_0$
    \For{$\ell = 1, \dots, L$}
      \For{$j = 1, \dots, m$}
        \State Obtain problem instance $\theta_*^j \sim \cP$
        \State Sample all rewards $Y^j \mid \theta_*^j, x_{1 : n}$
        \State Apply policy $\pi_w$ to contexts $x_{1 : n}$ and rewards $Y^j$, and obtain pulled arms $I^j$
      \EndFor
      \State Let $\hat{g}(n; \pi_w)$ be an estimate of $\nabla_w r(n; \pi_w)$ from $x_{1 : n}$, $(Y^j)_{j = 1}^m$, and $(I^j)_{j = 1}^m$
      \State $w \gets w + \alpha \, \hat{g}(n; \pi_w)$
    \EndFor
    \Statex
    \State \textbf{Output:} Learned policy parameters $w$
  \end{algorithmic}
\end{algorithm}

A key motivation for our work is the design of bandit policies that exploit problem-specific properties, similarly to Bayesian methods. However, we do no require a prior over problem instances, which can be difficult to obtain in many realistic problems. In this section, we present \gradband, a general data-dependent algorithm for learning bandit policies that assumes access to sampled problem instances (\cref{sec:framework}) and optimizes these policies using gradient ascent with respect to policy parameters. The algorithm is detailed in \cref{sec:gradband}. We analyze it in special cases in \cref{sec:theory} and compare it to related designs in \cref{sec:related designs}. Specific instances of \gradband for multi-armed and contextual bandits are presented in \cref{sec:mab algorithms,sec:contextual algorithms}, respectively.

\subsection{Algorithm \gradband}
\label{sec:gradband}

Our algorithm \gradband is shown in \cref{alg:gradband}. Assuming some parameterized class of bandit policies $\pi_w$, the key idea is to maximize the Bayes reward $r(n; \pi_w)$ using gradient ascent with respect to policy parameters $w$ on sampled problem instances from $\cP$. \gradband is initialized with policy parameters $w_0$. In each iteration, the parameters of the prior policy $w$ are updated by gradient ascent using $\hat{g}(n; \pi_w)$, an empirical estimate of the \emph{reward gradient}, $\nabla_w r(n; \pi_w)$, of that prior policy. To compute $\hat{g}(n; \pi_w)$, we run the policy $\pi_w$ on $m$ sampled problem instances, where $\theta_*^j$ denotes the $j$-th sampled instance, its realized rewards are $Y^j \in \realset^{K \times n}$, and its pulled arms are $I^j \in [K]^n$. We detail computation of the reward gradient in \cref{sec:reward gradient}.

The per-iteration time complexity of \gradband is $O(K m n)$, because it samples $m$ problem instances from $\cP$ at each iteration, each instance having horizon $n$ and $K$ arms, and executes one policy in each instance. As shown in \cref{sec:reward gradient}, only $x_{1 : n}$, $Y^j$, and $I^j$ are necessary to compute the empirical gradient $\hat{g}(n; \pi_w)$. Therefore, \gradband does not need to know the sampled model parameters $\theta_*^j$ or the prior distribution $\cP$ to optimize the Bayes reward. This approach to policy optimization is a form of policy gradients, and is common in reinforcement learning (\cref{sec:related designs}).

\subsection{Theory}
\label{sec:theory}

\gradband is general, data-dependent, and directly optimizes the Bayes reward $r(n; \pi)$ with respect to an implicit empirical estimate of the instance distribution $\cP$. However, since $r(n; \pi)$ is a complex function of the adaptive bandit policy $\pi$ and the environment, it is hard to provide meaningful theoretical guarantees on the optimality of the resulting policy. This is one reason why most existing work on bandit algorithms analyzes theoretically manageable regret upper bounds rather than the regret itself. We provide the first such guarantee below. Specifically, we show that the $n$-round Bayes reward of a randomized explore-then-commit policy in a $2$-armed Gaussian bandit is concave in its exploration horizon. As a result, \gradband enjoys the same convergence guarantees as gradient descent for convex functions.

\begin{restatable}[]{theorem}{concavity}
\label{thm:concavity} Consider a $2$-armed Gaussian bandit where the reward of arm $i \in [2]$ in round $t$ is $Y_{i, t} \sim \cN(\mu_i, 1)$. Consider an \emph{explore-then-commit policy} $\pi_h$ with parameter $h \in \cW = [1, \floors{n / 2}]$ that explores each arm $\bar{h} = \floors{h} + Z$ times for $Z \sim \mathrm{Ber}(h - \floors{h})$. Then for any prior distribution $\cP$ over the pair of arm means $\theta_* = (\mu_1,\mu_2) \in \realset^2$, the Bayes reward $r(n; \pi_h)$ is concave in $h$.
\end{restatable}

The theorem is proved in \cref{sec:concave bayes reward}. The key insight is that $r(n; \pi_h)$ under the explore-then-commit policy in a $2$-armed Gaussian bandit has a closed form that is differentiable with respect to $h$. The randomization in \cref{thm:concavity} is only needed to extend $\pi_h$ to continuous exploration horizons $h$. In this case, \gradband enjoys the same convergence guarantees as gradient descent for convex functions \citep[Section 9.3]{boyd04convex}.

We can readily generalize \cref{thm:concavity} to the contextual setting.

\begin{restatable}[]{theorem}{coconcavity}
\label{thm:contextual concavity} Consider a contextual $2$-armed Gaussian bandit with $L$ contexts, where the reward of arm $i$ in context $x \in [L]$ at any round $t$ is $Y_{i, t} \sim \cN(\mu_{i, x}, 1)$. Consider a contextual bandit policy $\pi_h$ that applies the explore-then-commit policy (\cref{thm:concavity}) \emph{separately} in each context. Then for any sequence of contexts $x_{1 : n} \in [L]^n$ and prior distribution $\cP$ over model parameters
\begin{align*}
  \theta_*
  = (\mu_{1, 1}, \mu_{2, 1}) \oplus \dots \oplus (\mu_{1, L}, \mu_{2, L})\,,
\end{align*}
the Bayes reward $r(n; \pi_h)$ is concave in $h$.
\end{restatable}

The theorem is proved in \cref{sec:concave bayes reward}. As in \cref{thm:concavity}, we rely on the fact that the expected $n$-round reward of the explore-then-commit policy in a $2$-armed Gaussian bandit has a closed form, for any context and problem instance.

Guarantees like these, and other empirical evidence \citep{hsu19empirical}, justify \emph{gradient-based optimization} of bandit policies. It is difficult to provide guarantees of this form in general. For this reason, we validate the convergence of policy-gradient optimization to good policies empirically in \cref{sec:experiments}.

\subsection{Related Designs}
\label{sec:related designs}

Our approach to meta-learning of bandit policies using policy gradients lies at the intersection of stochastic bandits and reinforcement learning. We briefly contrast these related designs below.

\paragraph{Stochastic multi-armed bandits:} Our objective, maximizing the Bayes reward $r(n; \pi)$ with respect to instance distribution $\cP$, is different from classical designs, which maximize the expected $n$-round reward $r(n, \theta_*; \pi)$ in \emph{any} problem instance $\theta_*$ \citep{lai85asymptotically,auer02finitetime,lattimore19bandit}. The latter is more demanding and guards against worst-case failures, high regret on some problem instance. Our objective is more natural when the prior $\cP$ can be estimated from data and poor performance on unlikely instances can be tolerated.

\paragraph{Bayesian bandits:} Early work on Bayesian bandits \citep{gittins79bandit,berry85bandit,gittins11multiarmed} focused on deriving Bayes optimal policies for specific conjugate priors $\cP$. While these are provably Bayes optimal, such policies must be crafted for specific families of priors, and are only optimal under such priors. By contrast, we make no such assumptions on $\cP$. However, we do lose Bayes optimality guarantees. In particular, the optimal policy may not take the form of the parameterized policy $\pi$ that \gradband optimizes. Since \gradband differentiates policies, it can be computationally costly (\cref{sec:gradband}). Nevertheless, it is less costly and easier to parallelize than the computation of typical Bayes optimal policies, as we detail in \cref{sec:comparison existing solutions}.

\paragraph{Reinforcement learning:} Our approach to learning a bandit policy $\pi$ can be viewed as an instance of \emph{reinforcement learning (RL)} \citep{sutton88learning} where the \emph{state} in round $t$ is the history $H_t$, the \emph{action} is the pulled arm $I_t$, and the \emph{reward} is the realized reward $Y_{I_t, t}$. The main challenge in applying RL methods directly is the fact that the number of dimensions in $H_t$ increases linearly with round $t$.\footnote{Our framing of the problem can be viewed as a \emph{partially-observable} RL problem. Indeed, under the full Bayesian formulation with a known prior $\cP$, the history can be summarized succinctly as a posterior over model parameters (or \say{belief state}), resolving the curse of dimensionality induced by the horizon. Since we assume no knowledge of the prior, this cannot be done in our setting.} Any RL method that solves this problem must introduce some structure to deal with the \emph{curse of dimensionality}. Since it is not clear what shape the value function might take in general, we opt for optimizing parametric bandit policies using policy gradients \citep{williams92simple}, a commonly used technique in RL. The main novelty in our application of policy gradients is the use of specific baseline subtraction techniques that are tailored to the bandit structure of our problem, as detailed in the next section.

%% file: Gradient.tex
\section{Reward Gradient}
\label{sec:reward gradient}

\gradband is a meta-algorithm that maximizes the Bayes reward $r(n; \pi_w)$ of policy $\pi_w$ by gradient ascent. To apply it, we need to choose a parameterized policy class $\pi_w$ and be able to compute the gradient of $r(n; \pi_w)$ with respect to policy parameters $w$. In this section, we derive the gradient for any policy $\pi_w$, its empirical approximation, and propose baseline subtraction to reduce its variance. We introduce specific differentiable policy classes in \cref{sec:mab algorithms,sec:contextual algorithms}.

\subsection{Reward Gradient Derivation}
\label{sec:reward gradient derivation}

The Bayes reward $r(n; \pi_w)$ has two main structural properties that we exploit in the derivation of its gradient. First, it is additive over rounds. Second, the reward in round $t$ does not depend on the actions taken, arms pulled, by the policy after that round. We rely on both properties to differentiate $r(n; \pi_w)$, which yields the following result.

\begin{restatable}[]{lemma}{gradient}
\label{lem:gradient} The gradient of the $n$-round Bayes reward of policy $\pi_w$ with respect to $w$ is
\begin{align*}
  \nabla_w r(n; \pi_w)
  = \sum_{t = 1}^n \E{\nabla_w \log \pi_w(I_t \mid H_t) \sum_{s = t}^n Y_{I_s, s}}\,.
\end{align*}
\end{restatable}

Here the expectation is taken with respect to problem instances, the arms pulled by the policy, and the realized rewards. The proof of \cref{lem:gradient} is provided in \cref{sec:gradient proofs}. The gradient is a sum of $n$ terms, one for each round $t$. The $t$-th term is itself is sum of $n - t + 1$ terms, reflecting the fact that a change in the policy in round $t$, $\nabla_w \log \pi_w(I_t \mid H_t)$, affects $n - t + 1$ future rewards, $\sum_{s = t}^n Y_{I_s, s}$. This implies that $\nabla_w r(n; \pi_w)$ involves an expectation over $O(n^2)$ random quantities. As a consequence, a naive empirical estimate of $\nabla_w r(n; \pi_w)$ is expected to have large variance, an issue that we address next.

\subsection{Baseline Subtraction}
\label{sec:baseline subtraction}

To reduce variance in reward gradients, we apply the idea of \emph{baseline subtraction} \citep{williams92simple,sutton00policy,greensmith04variance,munos06geometric,zhao11analysis,dick15policy,liu18actiondependent}, a standard concept in statistical Monte-Carlo estimation. Let
\begin{align*}
  b_t: [K]^{t - 1} \times \realset^{K n} \times \realset^{d n} \times \Theta \to \realset
\end{align*}
be any function of $t - 1$ previously pulled arms, all realized rewards, all contexts, and the problem instance itself. A \emph{baseline} is any collection of such functions $b = (b_t)_{t = 1}^n$, one per round $t$. As we show below, any baseline can be subtracted from realized rewards in the reward gradient without impacting it.

\begin{restatable}[]{lemma}{baselinegradient}
\label{lem:baseline gradient} For any baseline $b = (b_t)_{t = 1}^n$,
\begin{align*}
  \nabla_w r(n; \pi_w)
  = \sum_{t = 1}^n \E{\nabla_w \log \pi_w(I_t \mid H_t)
  \left(\sum_{s = t}^n Y_{I_s, s} - b_t(I_{1 : t - 1}, Y, x_{1 : n}, \theta_*)\right)}\,.
\end{align*}
\end{restatable}

The proof of \cref{lem:baseline gradient} in provided in \cref{sec:gradient proofs}, and relies on the observation that $b$ is independent of the future actions taken by $\pi_w$. An empirical approximation to the gradient based on $m$ sampled instances from $\cP$ can then incorporate baseline subtraction as
\begin{align}
  \hat{g}(n; \pi_w)
  = \frac{1}{m} \sum_{j = 1}^m \sum_{t = 1}^n
  \nabla_w \log \pi_w(I_t^j \mid H_t^j)
  \left(\sum_{s = t}^n Y_{I_s^j, s}^j -
  b_t(I_{1 : t - 1}^j, Y^j, x_{1 : n}, \theta_*^j)\right)\,,
  \label{eq:empirical baseline gradient}
\end{align}
where $j$ indexes the $j$-th random experiment in \gradband.

\subsection{Baselines}
\label{sec:baselines}

We now consider several specific baselines that can be used to reduce variance of empirical reward gradient estimates. These baselines are motivated by the structure of our bandit problem. Most bandit policies are designed to have sublinear regret in any problem instance \citep{lattimore19bandit}. This means that when the number of arms is finite, the policies pull optimal arms with increasing frequency. Therefore, one natural baseline in \cref{lem:baseline gradient} would be the sum of the rewards of the optimal arms,
\begin{align}
  b_t^\textsc{opt}(I_{1 : t - 1}, Y, x_{1 : n}, \theta_*)
  = \sum_{s = t}^n Y_{i_{*, s}(\theta_*), s}\,.
  \label{eq:baseline opt}
\end{align}
The optimal arm in round $s$, $i_{*, s}(\theta_*)$, is defined in \eqref{eq:n-round regret}. Note that the baseline $b^\textsc{opt}$ can be used only because of the specific form of our gradient (\cref{lem:baseline gradient}), where $b_t$ is a function of the same realized rewards, contexts, and model parameters as policy $\pi$.

Unfortunately, $b^\textsc{opt}$ may reduce variance poorly when the policy $\pi$ has high regret. In this case, subtracting the sum of the rewards of an \emph{independent run} of $\pi$ may be more effective. We call this baseline \say{self} and define it as
\begin{align*}
  b_t^\textsc{self}(I_{1 : t - 1}, Y, x_{1 : n}, \theta_*)
  = \sum_{s = t}^n Y_{J_s, s}\,,
\end{align*}
where $(J_t)_{t = 1}^n$ are the pulled arms in an independent run of the policy $\pi$. Critically, $b^\textsc{self}$ has the same expected reward as $\pi$, and thus performs well even if $\pi$ has high regret. On the other hand, when $\pi$ has low regret, $b^\textsc{self}$ is by definition comparable to pulling optimal arms, and also reduces variance effectively. Thus, we expect $b^\textsc{self}$ to always outperform $b^\textsc{opt}$.

\subsection{Regret-Minimizing Policies}
\label{sec:regret-minimizing policies}

The baseline $b^\textsc{opt}$ can be also used to prove that the reward gradient is small when the optimized bandit policy has low regret, a fact that we exploit later in the construction of meta-learned bandit policies. Specifically, we prove that $\normw{\nabla_w r(n; \pi_w)}{2} = o(n^2)$ when the regret of $\pi_w$ is sublinear in any problem instance $\theta_*$.

\begin{restatable}[]{lemma}{baselineopt}
\label{lem:baseline opt} Let the reward gradient $\nabla_w r(n; \pi_w)$ be defined as in \cref{lem:baseline gradient} and the baseline be $b^\textsc{opt}$ in \eqref{eq:baseline opt}. Let $\normw{\nabla_w \log \pi_w(I_t \mid H_t)}{2} \leq c$ hold for all $I_t$ and $H_t$. Let the $n$-round regret of $\pi_w$ in any problem instance $\theta_*$ be $R(n, \theta_*; \pi_w) = O(n^\alpha)$. Then
\begin{align*}
  \normw{\nabla_w r(n; \pi_w)}{2}
  \leq c \frac{\Delta_{\max}}{\Delta_{\min}} O(n^{\alpha + 1})\,,
\end{align*}
where
\begin{align*}
  \Delta_{\min}
  & = \min_{\theta_* \in \Theta} \min_{t \in [n]}
  \left[f_{i_{*, t}(\theta_*)}(x_t, \theta_*) -
  \max_{i \neq i_{*, t}(\theta_*)} f_i(x_t, \theta_*)\right]\,, \\
  \Delta_{\max}
  & = \max_{\theta_* \in \Theta} \max_{t \in [n]}
  \left[f_{i_{*, t}(\theta_*)}(x_t, \theta_*) -
  \max_{i \neq i_{*, t}(\theta_*)} f_i(x_t, \theta_*)\right]\,,
\end{align*}
are the minimum and maximum gaps, respectively, over all problem instances and rounds.
\end{restatable}


\cref{lem:baseline opt} is proved in \cref{sec:gradient proofs}. The fact that the gradient is small when the policy has low regret motivates a particular way of using our meta-learning framework. Specifically, \gradband is initialized with sound bandit policies that are then optimized to maximize the Bayes reward. As an example, consider the \say{soft elimination} policy that we introduce later in \cref{sec:softelim}. This policy has $O(\log n)$ regret in any problem instance (\cref{thm:softelim regret bound}). Therefore, $\normw{\nabla_w r(n; \pi_w)}{2} = O(n^{\alpha + 1})$ for $\alpha = \log \log n /  \log n$ and $w = \sqrt{8}$.

%% file: Algorithms.tex
\section{Differentiable Multi-Armed Bandit Algorithms}
\label{sec:mab algorithms}

Now we consider specific non-contextual policy classes that \gradband can optimize. The contextual case is examined in \cref{sec:contextual algorithms}. First we offer some general observations, then examine three policy classes: \expthree, \softelim, and recurrent neural network policies.

\subsection{General Considerations}
\label{sec:mab intro}

\cref{lem:baseline opt} shows that the reward gradient is small when the bandit policy has sublinear regret. This suggests that policy classes with provably sublinear regret are especially suitable for optimization by \gradband. In this way, we can exploit existing bandit algorithms in the literature, but improve their performance by optimizing their Bayes reward with respect to problem distribution $\cP$. We do require that these policy classes be parameterized, and hence for some choice of their parameters, the policy has sublinear regret. In such a case, \gradband is initialized with policy $\pi_{w_0}$ where $w_0$ is such a parameter choice.

A necessary condition for computing the reward gradient in \cref{lem:gradient} is that the gradient of the probability of pulling arm $i$ conditioned on history, $\nabla_w \log \pi_w(i \mid H_t)$, exists. Unfortunately, few classical bandit policies satisfy this assumption. For example, UCB algorithms \citep{auer02finitetime,dani08stochastic,abbasi-yadkori11improved} do not since $\pi_w(i \mid H_t) \in \set{0, 1}$ is a step function. Hence these cannot be used directly. In contrast, Thompson sampling \citep{thompson33likelihood,agrawal12analysis,agrawal13thompson} is randomized. However, $\pi_w(i \mid H_t)$ is induced by a strict maximization over randomized posterior means. Therefore, a unique gradient may not exist. Even when it does, $\pi_w(i \mid H_t)$ does not have a closed form and therefore may be hard to differentiate computationally efficiently. This presents challenges to the use of Thompson sampling in \gradband. We return to this issue in \cref{sec:ts}, where we derive $\nabla_w r(n; \pi_w)$ in an alternate manner.

In the rest of this section, we introduce three \emph{softmax} designs that can be differentiated analytically and derive a gradient for each. We note that other softmax policies, such as Boltzmann exploration \citep{sutton98reinforcement,cesabianchi17boltzmann}, can be differentiated similarly.

To simplify notation, we let $\pi_{i, t} = \pi_w(i \mid H_t)$. We also assume that all realized rewards are bounded on $[0, 1]$, and thus are $\sigma^2$-sub-Gaussian for $\sigma = 1 / 2$. Finally, we denote the mean reward of arm $i$ by $\mu_i$, and relate it to model parameters $\theta_*$ as $\mu_i = (\theta_*)_i$.

\subsection{Algorithm \expthree}
\label{sec:exp3}

\expthree \citep{auer95gambling} is a well-known algorithm for non-stochastic bandits. The algorithm pulls arm $i$ in round $t$ with probability
\begin{align}
  \pi_{i, t}
  = (1 - w) \frac{\exp[\eta S_{i, t}]}{\sum_{j = 1}^K \exp[\eta S_{j, t}]} +
  \frac{w}{K}\,,
  \label{eq:exp3 distribution}
\end{align}
where $S_{i, t} = \sum_{\ell = 1}^{t - 1} \I{I_\ell = i} \pi_{i, \ell}^{-1} Y_{i, \ell}$ is the inverse propensity score estimate of the cumulative reward of arm $i$ in the first $t - 1$ rounds, $\eta$ is a learning rate, and $w$ is a parameter that guarantees sufficient exploration. When rewards are in $[0, 1]$, $\expthree$ attains $O(\sqrt{K n})$ regret for $\eta = w / K$ and $w = \min \set{1, \sqrt{K \log K} / \sqrt{(e - 1) n}}$. In this work, we \emph{optimize} $w$ using \gradband. When $\eta$ is set as above, we obtain the following reward gradient.

\begin{restatable}[]{lemma}{expthreederivative}
\label{lem:exp3 derivative} Define $\pi_{i, t}$ as in \eqref{eq:exp3 distribution}. Let $\eta = w / K$, $V_{i, t} = \exp[w S_{i, t} / K]$, and $V_t = \sum_{j = 1}^K V_{j, t}$. Then
\begin{align*}
  \nabla_w \log \pi_{i, t}
  = \frac{1}{\pi_{i, t}} \left[\frac{V_{i, t}}{V_t}
  \left[(1 - w) \left[\frac{S_{i, t}}{K} -
  \sum_{j = 1}^K \frac{V_{j, t}}{V_t} \frac{S_{j, t}}{K}\right] - 1\right] +
  \frac{1}{K}\right]\,.
\end{align*}
\end{restatable}

This lemma is proved in \cref{sec:gradient proofs}. Although \expthree is differentiable, it tends to be relatively conservative in stochastic problems, even after we optimize $w$ (\cref{sec:experiments}).

\subsection{Algorithm \softelim}
\label{sec:softelim}

To address the overly conservative nature of \expthree, we propose a new parameterized class of bandit policies called \softelim. The algorithm works as follows. Let $\hat{\mu}_{i, t}$ be the empirical mean reward of arm $i$ after $t$ rounds and $T_{i, t}$ be the respective number of pulls. \softelim begins by pulling each arm once. Subsequently, in round $t > K$, arm $i$ is pulled with probability
\begin{align*}
  \pi_{i, t}
  = \frac{\exp[- S_{i, t} / w^2]}{\sum_{j = 1}^K \exp[- S_{j, t} / w^2]}\,,
\end{align*}
where
\begin{align}
  S_{i, t}
  = 2 \, (\max_{j \in [K]} \hat{\mu}_{j, t - 1} - \hat{\mu}_{i, t - 1})^2 T_{i, t - 1}
  \label{eq:softelim score}
\end{align}
is the \emph{score} of arm $i$ in round $t$ and $w > 0$ is a tunable \emph{exploration parameter}. Since $S_{i, t} \geq 0$ and $\pi_{i, t} \propto \exp[- S_{i, t} / w^2]$, higher values of $w$ lead to more exploration. Also note that $\exp[- S_{i, t} / w^2] \in [0, 1]$. As a result, \softelim can be viewed as an elimination algorithm \citep{auer10ucb} where $\exp[- S_{i, t} / w^2]$ measures the degree to which arm $i$ has been eliminated. Because an arm is never completely eliminated, the elimination is \say{soft,} and hence the name of the algorithm.

Since $\log \pi_{i, t} = - w^{-2} S_{i, t} - \log \sum_{j = 1}^K \exp[- S_{j, t} / w^2]$, we have
\begin{align*}
  \nabla_w \log \pi_{i, t}
  = 2 w^{-3} \left(S_{i, t} - \sum_{j = 1}^K S_{j, t}
  \frac{\exp[- S_{j, t} / w^2]}{\sum_{j = 1}^K \exp[- S_{j, t} / w^2]}\right)\,.
\end{align*}
Therefore, \softelim can be easily differentiated and optimized by \gradband.

\softelim is designed such that an arm is unlikely to be pulled if it has been pulled \say{often} and its empirical mean is low relative to the arm with the highest empirical mean. This follows from the definition of score $S_{i, t}$. Moreover, when a suboptimal arm has been pulled \say{often} and has the highest empirical mean, the optimal arm is pulled proportionally to how much its empirical mean deviates from the actual mean. This indicates that our algorithm has suitable optimism to admit a regret analysis, similarly to follow-the-perturbed-leader bandit algorithms \citep{kveton19perturbed,kveton19perturbed2}. We rely on this property in our analysis, the result of which is presented below.

\begin{restatable}[]{theorem}{softelimregretbound}
\label{thm:softelim regret bound} Consider a $K$-armed bandit problem where arm $1$ is optimal, that is $\mu_1 > \max_{i > 1} \mu_i$. Let $\Delta_i = \mu_1 - \mu_i$ and $w = \sqrt{8}$. Then
\begin{align*}
  R(n, \theta_*; \pi_w)
  \leq \sum_{i = 2}^K (2 e + 1) \left(\frac{16}{\Delta_i} \log n +
  \Delta_i\right) + 5 \Delta_i\,.
\end{align*}
\end{restatable}

\cref{thm:softelim regret bound} is proved in \cref{sec:softelim analysis}, which also includes a proof sketch. The value of $w = \sqrt{8}$ is obtained by tuning. A similar bound, with worse constants, can be derived for any $w \in (1, \sqrt{8}]$. This can be seen in the proof, which requires only that $\gamma = 1 / w^2 \in [1 / 8, 1)$. Finally, note that our regret bound scales with the gaps $\Delta_i$ and $\log n$ similar to that of \ucb \citep{auer02finitetime}. Thus \softelim is near-optimal.

\subsection{Recurrent Neural Network}
\label{sec:rnn}

Both \expthree and \softelim involve taking softmaxes over arm scores, which are essentially \say{hand-crafted} features that summarize the history of the policy. It is this careful feature selection which allows theoretical analyses. However, from a practical perspective, it is also possible to \emph{learn the features} that summarize this history by a \emph{recurrent neural network (RNN)} \citep{rumelhart86learning,hochreiter97long}.

We devise a novel class of RNN bandit policies that work as follows. We assume an RNN that summarizes the history as a state. In round $t$, the RNN takes its prior state $s_{t - 1}$, the pulled arm $I_t$, and its realized reward $Y_{I_t, t}$ as inputs. Then it updates the state into a new state $s_t$, and outputs the probability $\pi_{i, t + 1}$ of pulling any arm $i$ in round $t + 1$,
\begin{align*}
  s_t
  = \textsc{RNN}_\Phi(s_{t - 1}, (I_t, Y_{I_t, t}))\,, \quad
  \pi_{i, t + 1}
  = \frac{\exp[v_i\T s_t]}{\sum_{j = 1}^K \exp[v_j\T s_t]}\,.
\end{align*}
The update is parameterized by $w = (\Phi, \set{v_i}_{i = 1}^K)$, where $\Phi$ are the RNN parameters and $v_i$ are per-arm parameters. It is these parameters that we optimize using \gradband. The key insight behind this policy class is that the RNN can learn to summarize the history of the policy using its internal state $s_t$, that adapts to a specific problem structure uncovered in the prior distribution $\cP$. Indeed, such policies demonstrate the generality and flexibility of \gradband.

Details of our particular RNN implementation can be found in \cref{sec:rnn implementation}. We use an LSTM \citep{hochreiter97long} with a $d$-dimensional latent state as our RNN architecture. We assume that the rewards are Bernoulli and the initial state is $s_0 = \mathbf{0}$.

%% file: AlgorithmsCo.tex
\section{Differentiable Contextual Bandit Algorithms}
\label{sec:contextual algorithms}

In this section, we generalize the ideas of \cref{sec:mab algorithms} to contextual bandits. We pose the problem of learning a contextual bandit algorithm as learning a \emph{projection of contexts} into a relevant subspace where we employ a linear bandit algorithm. This design is motivated by three observations. First, uncertainty in linear models is relatively well understood \citep{dani08stochastic,abbasi-yadkori11improved}. Second, state-of-the-art approaches in meta-learning of linear models project features into a relevant subspace to speed up learning \citep{bullins19generalize,tripuraneni20provable}. Finally, our early experiments with recurrent neural networks (\cref{sec:rnn experiment}) revealed that it is difficult to learn contextual bandit policies without additional structural assumptions, such as in this section.

As in \cref{sec:mab algorithms}, we first offer some general observations, then examine several policy classes that \gradband can optimize. In \cref{sec:cosoftelim}, we generalize \softelim (\cref{sec:softelim}) to contextual bandits. \cref{sec:ts} is devoted to Thompson sampling, which is arguably the most practical bandit algorithm. Finally, we differentiate the $\eps$-greedy policy in \cref{sec:e-greedy policy}. As in \cref{sec:mab algorithms}, to simplify notation, we let $\pi_{i, t} = \pi_w(i \mid H_t)$.

\subsection{General Considerations}
\label{sec:contextual intro}

All of our contextual bandit policies have the following structure. They rely on a projection matrix $W \in \realset^{d \times d}$ to project the context in round $t$, $x_t$, to a relevant subspace via $W x_t$. \gradband is used to \emph{optimize} $W$ to maximize the Bayes reward of \emph{learning to act in the induced subspace}. To stress that the optimized parameters form a matrix, we write $W$ instead of $w$. Learning of $W$ can help in many scenarios. For instance, if the $i$-th feature in the context vector is irrelevant for determining reward, we might expect \gradband to learn to zero out the $i$-th column of $W$. Once projected, we assume that the mean arm reward is linear in $W x_t$. Thus, in this induced subspace, the \emph{maximum likelihood estimate (MLE)} of the model parameters of arm $i$ after $t$ observations is
\begin{align}
  \hat{\theta}_{i, t}
  = G_{i, t}^{-1} \sum_{\ell = 1}^t \I{I_\ell = i} W x_\ell Y_{i, \ell}\,, \quad
  G_{i, t}
  = \sum_{\ell = 1}^t \I{I_\ell = i} W x_\ell x_\ell\T W\T + \lambda I_d\,,
  \label{eq:mle}
\end{align}
where $G_{i, t}$ is a \emph{sample covariance matrix} and $\lambda > 0$ is a regularization parameter. In round $t$, the learning agent acts based on its past $t - 1$ observations. Thus the mean reward of arm $i$ in context $x$ is $(W x)\T \hat{\theta}_{i, t - 1}$ and the variance of this estimate is $\normw{W x}{G_{i, t - 1}^{-1}}^2 = (W x)\T G_{i, t - 1}^{-1} W x$.

Intuitively, this design learns to project context into a subspace such that the transformed problem is a lower-dimensional linear bandit. In the remainder of this section, we propose three differentiable contextual bandit policies and derive their reward gradients. The gradients can be directly used in the empirical reward gradient in \eqref{eq:empirical baseline gradient}.

\subsection{Contextual Soft Elimination}
\label{sec:cosoftelim}

We first propose \cosoftelim, a contextual softmax or \say{soft elimination} policy, which generalizes the \softelim policy in \cref{sec:softelim} to contextual bandits. \cosoftelim pulls arm $i$ in round $t$ with probability
\begin{align}
  \pi_{i, t}
  = \frac{\exp[- S_{i, t}]}{\sum_{j = 1}^K \exp[- S_{j, t}]}\,,
  \label{eq:softmax}
\end{align}
where $S_{i, t} \geq 0$ is the \emph{score} of arm $i$ in round $t$, which depends on history $H_t$ that includes context $x_t$, and projection matrix $W$. Much like \softelim, because $\exp[- S_{i, t}] \in [0, 1]$, \cosoftelim can be viewed as a \say{soft} elimination algorithm, where $S_{i, t}$ reflects the degree to which arm $i$ has been eliminated.

The score of arm $i$ in round $t$ is defined as
\begin{align}
  S_{i, t}
  = \frac{\gamma (\hat{\mu}_{\max, t} - \hat{\mu}_{i, t})^2}
  {\normw{W x_t}{G_{i, t - 1}^{-1}}^2}\,,
  \label{eq:cosoftelim score}
\end{align}
where $\hat{\mu}_{i, t} = (W x_t)\T \hat{\theta}_{i, t - 1}$ is the estimated reward of arm $i$, $I_{\max, t} = \argmax_{i \in [K]} \hat{\mu}_{i, t}$ is the arm with the highest reward, $\hat{\mu}_{\max, t}$ is its reward, and $\gamma > 0$. Thus $S_{i, t}$ in \eqref{eq:cosoftelim score} is a natural generalization of \eqref{eq:softelim score}. Specifically, $S_{i, t}$ in \eqref{eq:softelim score} is a ratio of two quantities, the squared empirical suboptimality gap of arm $i$ and the variance of the mean reward estimate of arm $i$, $1 / T_{i, t - 1}$. The score $S_{i, t}$ in \eqref{eq:cosoftelim score} is also function of these quantities, but specialized to linear models. It ensures that arm $i$ is unlikely to be pulled if its empirical gap is large relative to the uncertainty in the direction of that arm. This exploration scheme has sublinear regret, as we prove below, and performs well in practice.

The conditional probability of pulling arm $i$ given history $H_t$, $\pi_{i, t}$, can be differentiated with respect to the projection matrix $W$ as
\begin{align*}
  \nabla_W \log \pi_{i, t}
  = - \nabla_W S_{i, t} - \nabla_W \log \sum_{j = 1}^K \exp[- S_{j, t}]
  = \frac{\sum_{j = 1}^K \exp[- S_{j, t}] \nabla_W S_{j, t}}
  {\sum_{j = 1}^K \exp[- S_{j, t}]} - \nabla_W S_{i, t}\,.
\end{align*}
In our experiments, we compute $\nabla_W S_{i, t}$ using automatic differentiation in TensorFlow. Note that $S_{i, t}$ depends only on history $H_t$ and projection matrix $W$. We experiment with $\gamma = 1 / \sigma^2$, where $\sigma$ is the sub-Gaussian noise parameter of rewards. A theory-justified $\gamma = \tilde{O}(1 / (\sigma^2 K d))$ is derived in our analysis below.

Now we turn to bounding the regret of \cosoftelim, under the assumption of a linear model. Specifically, we assume, for any arm $i$ and context $x$, that $f_i(x, \theta_*) = x\T \theta_{i, *}$ holds for some fixed unknown $\theta_{i, *} \in \realset^d$. The joint parameter vector is $\theta_* = \theta_{1, *} \oplus \dots \oplus \theta_{K, *}$. Our bound holds for any $\theta_*$ and contexts $x_{1 : n}$. To simplify exposition, we set $W = I_d$, which is equivalent to removing $W$ from \eqref{eq:mle}. Let $L = \max_{t \in [n]} \normw{x_t}{2}$ be the maximum length of a feature vector, $L_* = \normw{\theta_*}{2}$ be the parameter vector length, and $\Delta_{\max}$ be the maximum gap. Then we have the following result.

\begin{restatable}[]{theorem}{cosoftelimregretbound}
\label{thm:cosoftelim regret bound} Consider a contextual $K$-armed bandit where $f_i(x, \theta_*) = x\T \theta_{i, *}$ holds for some fixed unknown $\theta_{i, *} \in \realset^d$. Then for any $\lambda \geq L^2$ in \eqref{eq:mle} and $\gamma = c_1^{-2}$ in \eqref{eq:cosoftelim score}, we have
\begin{align*}
  R(n, \theta_*; \pi_W)
  \leq (12 K e + 3) c_1 \sqrt{c_2 n \log n} + (K + 1) \Delta_{\max}\,,
\end{align*}
where $c_1 = \tilde{O}(\sqrt{K d})$ (\cref{lem:concentration} in \cref{sec:cosoftelim analysis}) and $c_2 = \tilde{O}(K d)$ (\cref{lem:sum of squared confidence widths} in \cref{sec:cosoftelim analysis}).
\end{restatable}

\cref{thm:cosoftelim regret bound} is proved in \cref{sec:cosoftelim analysis}, which also includes a proof sketch. Our regret bound is $\tilde{O}(K^2 d \sqrt{n})$ and has optimal dependence on the number of rounds $n$. As the number of features is $K d$, \lints \citep{agrawal13thompson} has $\tilde{O}(K^\frac{3}{2} d^\frac{3}{2} \sqrt{n})$ regret in our setting. This means that our bound is tighter for $K < d$. The difference is due to the fact that Thompson sampling adds noise in $K d$ directions, to each entry of the estimate of $\theta_*$. By contrast, \cosoftelim only samples from a softmax over $K$ arms. \linucb \citep{abbasi-yadkori11improved} has $\tilde{O}(K d \sqrt{n})$ regret in our setting. So our bound is worse by a factor of $K$. This does not mean that \cosoftelim would be worse in practice. Our bound is worse because of a conservative analysis that relies on over-exploration, as in other randomized designs, such as \lints \citep{agrawal13thompson}.

\subsection{Contextual Thompson Sampling}
\label{sec:ts}

\emph{Thompson sampling (TS)} \citep{thompson33likelihood,chapelle11empirical,agrawal12analysis} is a state-of-the-art randomized bandit algorithm and \lints \citep{agrawal13thompson} generalizes it to the contextual setting. We adapt \lints to our linear model in \cref{sec:contextual intro} as follows. In round $t$, the estimated reward of arm $i$ is sampled as
\begin{align*}
  \tilde{\mu}_{i, t}
  \sim \mathcal{N}((W x_t)\T \hat{\theta}_{i, t - 1},
  \sigma^2 \normw{W x_t}{G_{i, t - 1}^{-1}}^2)\,.
\end{align*}
Then we pull the arm with the highest estimated reward $I_t = \argmax_{i \in [K]} \tilde{\mu}_{i, t}$.

As discussed in \cref{sec:mab intro}, it is hard to compute $\nabla_W \log \pi_{I_t, t}$ in Thompson sampling because the log probability is a result of an argmax over posterior-sampled arm means $\tilde{\mu}_{i, t}$, which does not have a closed form. To address this issue, we derive the reward gradient $\nabla_W r(n; \pi_W)$ below with explicitly included posterior samples $\tilde{\mu}_{i, t}$.

\begin{restatable}[]{lemma}{tsgradient}
\label{lem:ts gradient} The gradient of the $n$-round Bayes reward of Thompson sampling is
\begin{align*}
  \nabla_W r(n; \pi_W)
  = \sum_{t = 1}^n \E{\left(\sum_{i = 1}^K
  \nabla_W \log p_i(\tilde{\mu}_{i, t} \mid H_t; W)\right)
  \left(\sum_{s = t}^n Y_{I_s, s} - b_t(I_{1 : t - 1}, Y, x_{1 : n}, \theta_*)\right)}\,,
\end{align*}
where $b = (b_t)_{t = 1}^n$ is any baseline and $p_i(\cdot \mid H_t; W)$ is the posterior distribution of arm $i$ in round $t$. That is, $\tilde{\mu}_{i, t} \sim p_i(\cdot \mid H_t; W)$ and $I_t = \argmax_{i \in [K]} \tilde{\mu}_{i, t}$.
\end{restatable}

\cref{lem:ts gradient} is proved in \cref{sec:gradient proofs}. Interestingly, the gradient has the same algebraic form as that in \cref{lem:baseline gradient}, except that
\begin{align*}
  \nabla_W \log \pi_W(I_t \mid H_t)
  = \sum_{i = 1}^K \nabla_W \log \mathcal{N}(\tilde{\mu}_{i, t};
  (W x_t)\T \hat{\theta}_{i, t - 1}, \sigma^2 \normw{W x_t}{G_{i, t - 1}^{-1}}^2)\,,
\end{align*}
where $\mathcal{N}(\cdot; \mu, s^2)$ is a normal density with mean $\mu$ and variance $s^2$. To the best of our knowledge, this is the first derivation of the reward gradient for TS. Given the practical importance of TS, this is a major result.

We note that $\nabla_W \log \pi_W(I_t \mid H_t)$ in \cref{lem:baseline gradient} depends only on a single random pulled arm in round $t$, $I_t$; while the corresponding quantity in \cref{lem:ts gradient}, $\sum_{i = 1}^K \nabla_w \log p_i(\tilde{\mu}_{i, t} \mid H_t; W)$, depends on $K$ random posterior-sampled arm means $\tilde{\mu}_{i, t}$. The additional randomness due to $\tilde{\mu}_{i, t}$ is expected to increase the variance of reward gradients, which we observe empirically (\cref{sec:experiments}).

\subsection{$\eps$-Greedy Policy}
\label{sec:e-greedy policy}

The $\eps$-greedy policy \citep{sutton98reinforcement,auer02finitetime} can be easily applied to any generalization model and thus is popular in practice. The policy pulls arm $i$ in round $t$ with probability
\begin{align*}
  \pi_{i, t}
  = (1 - \eps) \I{I_{\max, t} = i} + \frac{\eps}{K}\,,
\end{align*}
where $\eps \in [0, 1]$ is the exploration rate of the policy and $I_{\max, t}$ is the arm with the highest estimated reward in round $t$. Unfortunately, $\I{I_{\max, t} = i}$ is a step function. Thus, even if the generalization model was parameterized, such as in \cref{sec:contextual intro} with $W$, the indicator would not be differentiable with respect to those parameters. Nevertheless, $\pi_{i, t}$ can be still differentiated with respect to $\eps$ as
\begin{align*}
  \nabla_\eps \log \pi_{i, t}
  = \frac{1}{\pi_{i, t}} \nabla_\eps \pi_{i, t}
  = \frac{1}{\pi_{i, t}} \left(\frac{1}{K} - \I{I_{\max, t} = i}\right)\,.
\end{align*}
This gradient depends only on scalar $\eps$ and can be used in \cref{lem:baseline gradient}, where $\eps$ plays the role of $w$.

We use this policy as a baseline in our experiments, to show the benefits of learning the projection matrix $W$. The generalization model is the same as in \eqref{eq:mle}, where we set $W = I_d$.

%% file: Experiments.tex
\section{Experiments}
\label{sec:experiments}

We conduct a number of experiments to demonstrate the generality and efficacy of \gradband. We also examine the performance of various policy classes proposed in \cref{sec:mab algorithms,sec:contextual algorithms}.

The first six experiments are non-contextual and showcase the policies from \cref{sec:mab algorithms}. In \cref{sec:reward gradient experiment}, we study reward gradients and their variance in a simple problem. In \cref{sec:policy optimization experiment,sec:comparison existing solutions}, we optimize \expthree and \softelim policies in this problem. In \cref{sec:more complex problems,sec:robustness experiment}, we experiment with more complex models and study the robustness of \gradband to its parameters. In \cref{sec:rnn experiment}, we evaluate the optimization of RNN policies. The last three experiments are contextual and showcase the algorithms from \cref{sec:contextual algorithms}. In \cref{sec:contextual policy optimization experiment}, we study the benefit of baseline subtraction. In \cref{sec:subspace meta-learning experiment}, we demonstrate meta-learning of relevant subspaces. Finally, in \cref{sec:multi-class classification experiment}, we apply our policies to real-world classification problems.

The performance of policies is measured by their Bayes regret. We use the Bayes regret instead of the Bayes reward because the former directly indicates how near-optimal a policy is. We note that lower Bayes regret implies higher Bayes reward, and vice versa. In all experiments, the Bayes regret is estimated from $1\,000$ i.i.d.\ problem instances sampled from $\cP$, which are \emph{independent} of those used by \gradband to optimize the policy in question. The shaded areas in our plots indicate standard errors in our estimates. All experiments are implemented in TensorFlow \citep{tensorflow} and PyTorch \citep{pytorch}, on a computer with $112$ cores and $392$ GB RAM.

\subsection{Reward Gradient}
\label{sec:reward gradient experiment}

\begin{figure*}[t]
  \centering
  \includegraphics[width=6in]{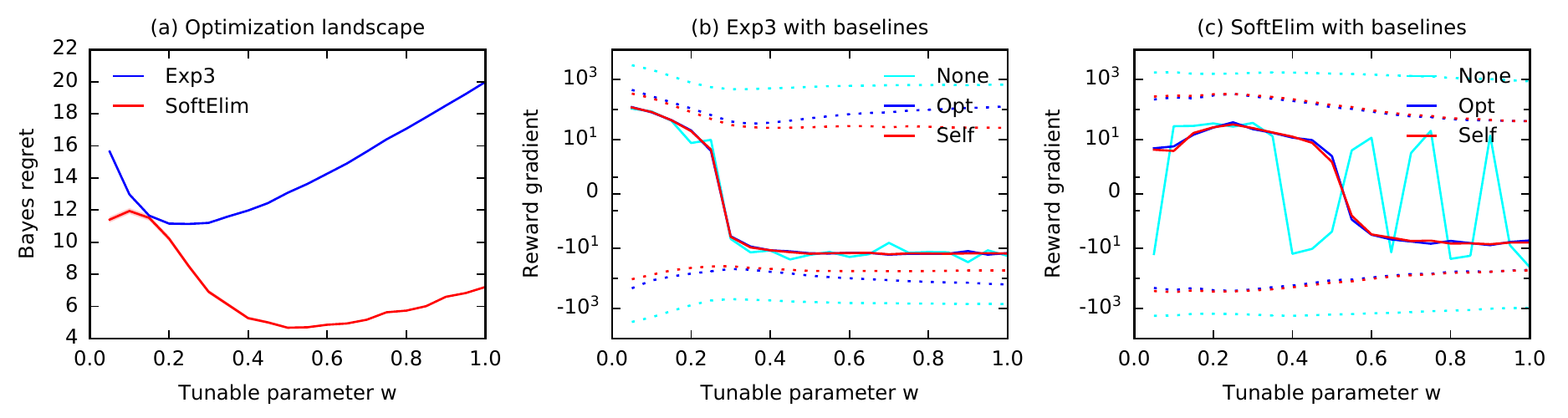} \\
  \vspace{-0.1in}
  \caption{The Bayes regret and reward gradients of \expthree and \softelim policies. In the last two plots, the solid lines are estimated reward gradients for $m = 10\,000$ in \eqref{eq:empirical baseline gradient} and the dotted lines mark high-probability regions for $m = 1$.}
  \label{fig:gradients}
\end{figure*}

Our first experiment is on a simple Bayesian bandit with $K = 2$ arms and Bernoulli reward distributions. The prior distribution $\cP$ is a mixture of two symmetric problem instances,
\begin{align*}
  \cP(\theta_*)
  = \frac{1}{2} \I{\theta_* = (0.6, 0.4)} + \frac{1}{2} \I{\theta_* = (0.4, 0.6)}\,,
\end{align*}
each with prior probability $0.5$. The horizon is $n = 200$ rounds. This is essentially a latent variable problem, where the bandit policy $\pi_w$ attempts to identify the optimal arm without knowing which instance it interacts with. In this case, maximizing \eqref{eq:bayes reward} with respect to $w$ is akin to attaining this as fast as statistically possible, without incurring high regret. The symmetry of the problem instances in this example is not instrumental.

The Bayes regret of \expthree (\cref{sec:exp3}) and \softelim (\cref{sec:softelim}) is shown in \cref{fig:gradients}a, as a function of their parameter $w$. Both are unimodal in $w$ and suitable for optimization by \gradband. \softelim has lower regret than \expthree for all values of $w\in [0,1]$. In fact, the minimum \emph{tuned} regret of \expthree is higher than that of \softelim without tuning. The untuned default is at $w = 1$.

Next we examine empirical gradients of the Bayes reward. The gradients of \expthree are reported in \cref{fig:gradients}b, as estimated using \eqref{eq:empirical baseline gradient}. When the number of samples in the estimator is large ($m = 10\,000$), all three baselines (\cref{sec:baselines}) yield similar gradients (solid lines in \cref{fig:gradients}b). When the number of samples is small ($m = 1$), baselines $b^\textsc{opt}$ and $b^\textsc{self}$ lead to an order of magnitude lower variance than no baseline substraction $b^\textsc{none}$. In particular, the dotted lines in \cref{fig:gradients}b show high-probability regions of the gradient estimates for $m = 1$. The gradients with baselines $b^\textsc{opt}$ and $b^\textsc{self}$ are generally between $-100$ and $100$, while those with $b^\textsc{none}$ are between $-1\,000$ and $1\,000$.

The gradients of \softelim are reported in \cref{fig:gradients}c. We observe similar trends to \cref{fig:gradients}b. The main difference is that $b^\textsc{none}$ performs even worse. Even when the number of samples in \eqref{eq:empirical baseline gradient} is large ($m = 10\,000$), the gradients with no baseline subtraction are highly noisy.

\subsection{Policy Optimization}
\label{sec:policy optimization experiment}

\begin{figure*}[t]
  \centering
  \includegraphics[width=6in]{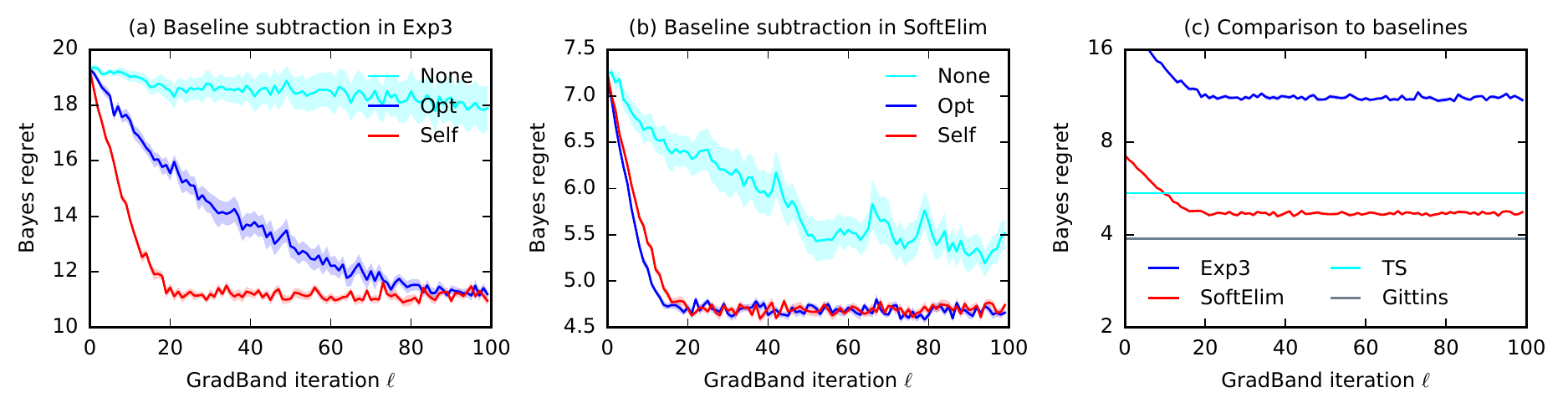} \\
  \vspace{-0.1in}
  \caption{The Bayes regret of \expthree and \softelim policies, as a function of \gradband iterations. All experiments are on a Bernoulli bandit with $K = 2$ arms (\cref{sec:policy optimization experiment}). The regret is averaged over $20$ runs.}
  \label{fig:policy optimization}
\end{figure*}

In the second experiment, we optimize \expthree and \softelim by \gradband on the simple problem from \cref{sec:reward gradient experiment}. The parameters of \gradband are set as $w_0 = 1$, $L = 100$ iterations, learning rate $\alpha = c^{-1} L^{- \frac{1}{2}}$, and batch size $m = 1\,000$. We select $c$ automatically so that $\norm{\hat{g}(n; \pi_{w_0})} \leq c$ holds with a high probability, to avoid manual tuning of the learning rate in our experiments.

In \cref{fig:policy optimization}a, we show the results of optimizing \expthree with all baselines. With $b^\textsc{self}$, \gradband learns a near-optimal policy in fewer than $20$ iterations. This is consistent with \cref{fig:gradients}b, where $b^\textsc{self}$ has the lowest variance. \cref{fig:policy optimization}b shows the equivalent results for \softelim. In this case, performance with $b^\textsc{opt}$ and $b^\textsc{self}$ is comparable. This consistent with \cref{fig:gradients}c, where the variances of $b^\textsc{opt}$ and $b^\textsc{self}$ are comparable. We conclude that $b^\textsc{self}$ is the best baseline overall and use it in all remaining experiments.

\subsection{Comparison to Existing Solutions}
\label{sec:comparison existing solutions}

To evaluate the quality of the policies learned by \gradband in \cref{sec:policy optimization experiment}, we compare them to four well-known bandit policies: \ucb \citep{auer02finitetime}, Bernoulli \ts \citep{agrawal12analysis} with $\mathrm{Beta}(1, 1)$ prior, \ucbv \citep{audibert09exploration}, and the Gittins index \citep{gittins79bandit}. These benchmarks are ideal points of comparison: (i) \ucb is arguably the most popular multi-armed bandit algorithm for $[0, 1]$ rewards. (ii) Bernoulli \ts is near-optimal for Bernoulli bandits, which is the problem structure in the support of our prior above. (iii) \ucbv adapts the sub-Gaussian parameter of reward distributions to observed rewards. This is akin to optimizing $w$ in \softelim. (iv) The Gittins index is the optimal solution to our problem, if the mean arm rewards were drawn i.i.d.\ from $\mathrm{Beta}(1, 1)$. Finally, we also compare to the Dopamine \citep{dopamine} implementation of a deep Q-learning RL algorithm DQN \citep{mnih13playing}, where the state is a concatenation of the following statistics for each arm: the number of observed ones, the number of observed zeros, the logarithm of both counts incremented by $1$, the empirical mean, and a constant bias term.

The Bayes regret of our baselines over $n = 200$ rounds is, respectively, $9.95 \pm 0.03$ (\ucb), $5.47 \pm 0.05$ (\ts), $15.79 \pm 0.03$ (\ucbv), $3.89 \pm 0.07$ (Gittins index), and $16.81 \pm 1.05$ (DQN). The Bayes regret of \softelim is $4.74 \pm 0.03$ (\cref{fig:policy optimization}b), and falls between that of \ts and the Gittins index. The regret of \expthree is $10.96 \pm 0.16$ (\cref{fig:policy optimization}a), which is not competitive. We plot the regret of the best benchmarks, the Gittins index and \ts, together with our policies in \cref{fig:policy optimization}c. This plot shows that a mere $10$ iterations of \gradband with \softelim policy are sufficient to learn a better policy than \ts. 

We conclude that tuned \softelim outperforms a strong baseline, \ts, and  performs almost as well as the Gittins index. We note that while the Gittins index is optimal in some limited settings, such as Bernoulli rewards, it is computationally costly. For instance, it took us almost two days to compute the Gittins index for horizon $n = 200$. This stands in a stark contrast with the tuning of \softelim by \gradband, which takes about $20$ seconds.

Now we discuss failures of some of the benchmarks. \ucbv fails because its variance optimism induces too much initial exploration. This is harmful for the somewhat short learning horizons in our experiment. DQN policies are unstable and need significant tuning to learn policies that barely outperform random actions. This is in a stark contrast with the simplicity of \gradband, which can learn near-optimal policies using vanilla gradient ascent. In the remaining experiments, we report only most competitive benchmarks from this experiment, the Gittins index and \ts.

\subsection{More Complex Problems}
\label{sec:more complex problems}

Now we apply \gradband to three more bandit problems. The first is a beta bandit, where rewards of arm $i$ are drawn from $\mathrm{Beta}(v (\theta_*)_i, v (1 - (\theta_*)_i))$ distribution and their variance is controlled by $v = 4$. The rest of the problem is set as in \cref{sec:policy optimization experiment}. The other two problems are variants of our Bernoulli and beta bandits, where the number of arms is $K = 10$, the prior distribution is defined as $\cP(\theta_*) = \prod_{i = 1}^K \mathrm{Beta}((\theta_*)_i; 1, 1)$, and the horizon is $n = 1\,000$ rounds. These problems are harder variants of our earlier problems, which have $2$ arms and a fixed gap. To apply \ts to $[0, 1]$ rewards, we use randomized Bernoulli rounding \citep{agrawal12analysis}.

\begin{figure*}[t]
  \centering
  \includegraphics[width=6in]{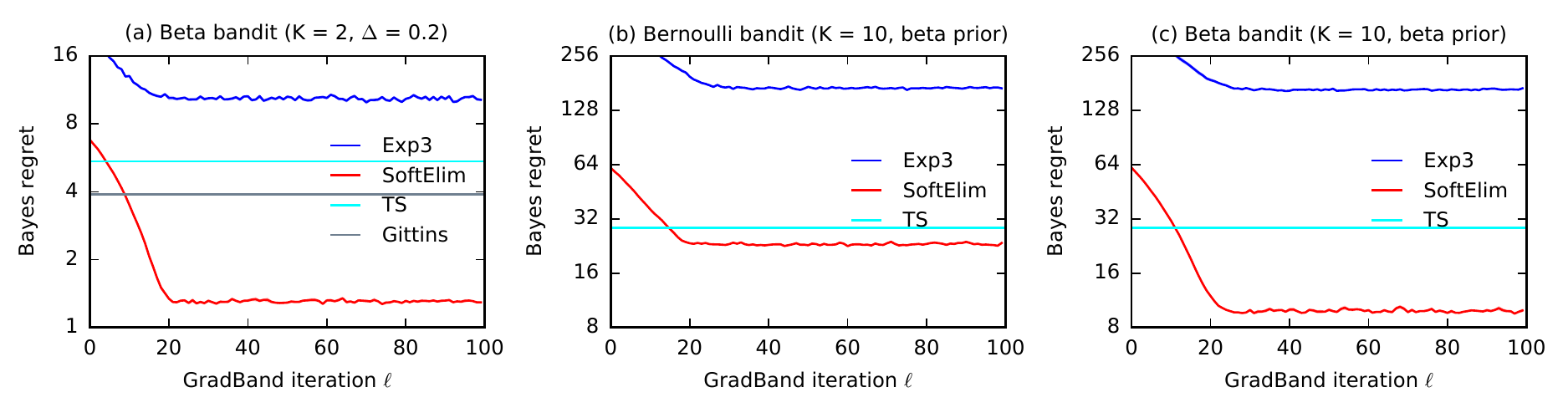} \\
  \vspace{-0.1in}
  \caption{The Bayes regret of \expthree and \softelim policies, as a function of \gradband iterations. The regret is averaged over $20$ runs.}
  \label{fig:large-scale policy optimization}
\end{figure*}

The regret of our policies is reported in \cref{fig:large-scale policy optimization}. In all problems, the regret of \softelim is lower than that of \ts, which is a competitive baseline. The greatest performance gains are in beta bandits, where \softelim adapts to low-variance rewards. \ts performs poorly due to the Bernoulli rounding, which replaces low-variance beta rewards with high-variance Bernoulli rewards. As observed earlier, tuned \expthree is not competitive.

\subsection{Robustness to Model and Algorithm Parameters}
\label{sec:robustness experiment}

In this section, we study the robustness of \gradband to its parameters and model misspecification. We conduct three experiments on the larger Bernoulli and beta problems in \cref{sec:more complex problems}.

\begin{figure*}[t]
  \centering
  \includegraphics[width=6in]{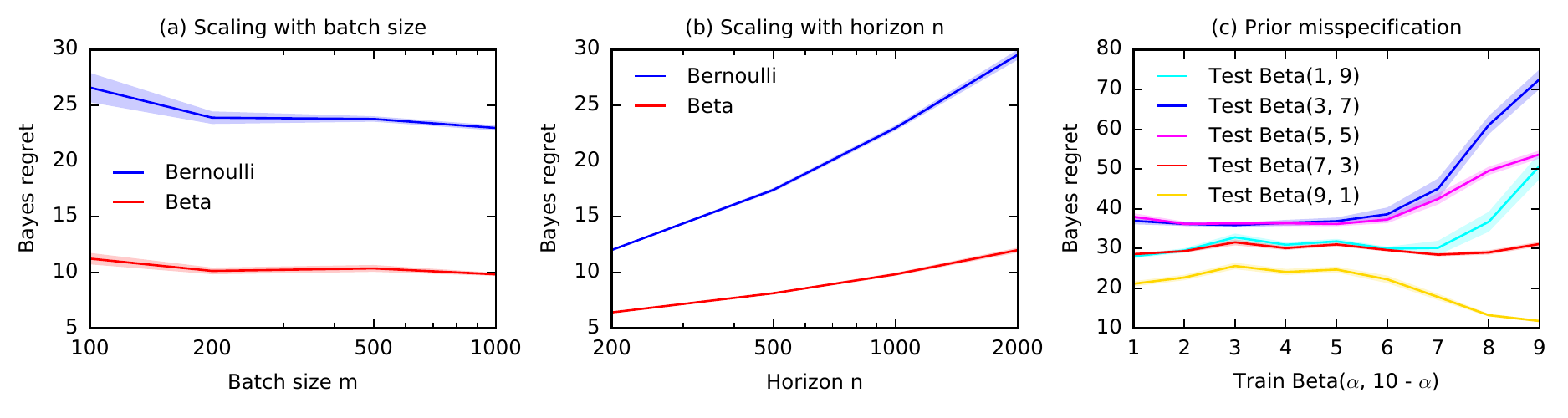} \\
  \vspace{-0.1in}
  \caption{The Bayes regret of \softelim policies as a function of batch size $m$ in \gradband, horizon $n$, and prior distribution $\cP$.}
  \label{fig:robustness}
\end{figure*}

In \cref{fig:robustness}a, we report the $n$-round regret of tuned \softelim as a function of batch size $m$ in \gradband. We observe that the regret is relatively stable as we decrease the batch size from $1\,000$ to $100$, but as expected the variance increases. Setting $m = 100$ reduces the run time of \gradband ten fold, when compared to $m = 1\,000$ used in our earlier experiments.

In \cref{fig:robustness}b, we report the $n$-round regret of tuned \softelim as we vary horizon $n$ from $200$ to $2\,000$ rounds. The regret is roughly linear in $\log n$, indicating theoretically optimal scaling. We expect this trend when the variance of reward gradients does not dominate \gradband. In this case, \gradband can optimize policies equally well at both shorter and longer horizons.

In \cref{fig:robustness}c, we investigate the robustness of tuned \softelim to prior misspecification. To do so, we tune \softelim on a Bernoulli bandit with prior $\cP(\theta_*) = \prod_{i = 1}^K \mathrm{Beta}((\theta_*)_i; \alpha, 10 - \alpha)$ for $\alpha \in [9]$, but measure its regret on a Bernoulli bandit with another $\alpha$. Unsurprisingly, the regret is larger when we train and test on different priors, compared to using the same prior for training and testing. For instance, when we train and test on $\alpha = 1$, the regret is about $30$. However, when we train on $\alpha = 9$, the regret becomes about $50$. Nevertheless, we do not observe catastrophic failures, that is an order of magnitude increase in regret. We conclude that \gradband is relatively robust to prior misspecification.

\subsection{RNN Policies}
\label{sec:rnn experiment}

\begin{figure*}[t]
  \centering
  \includegraphics[width=6in]{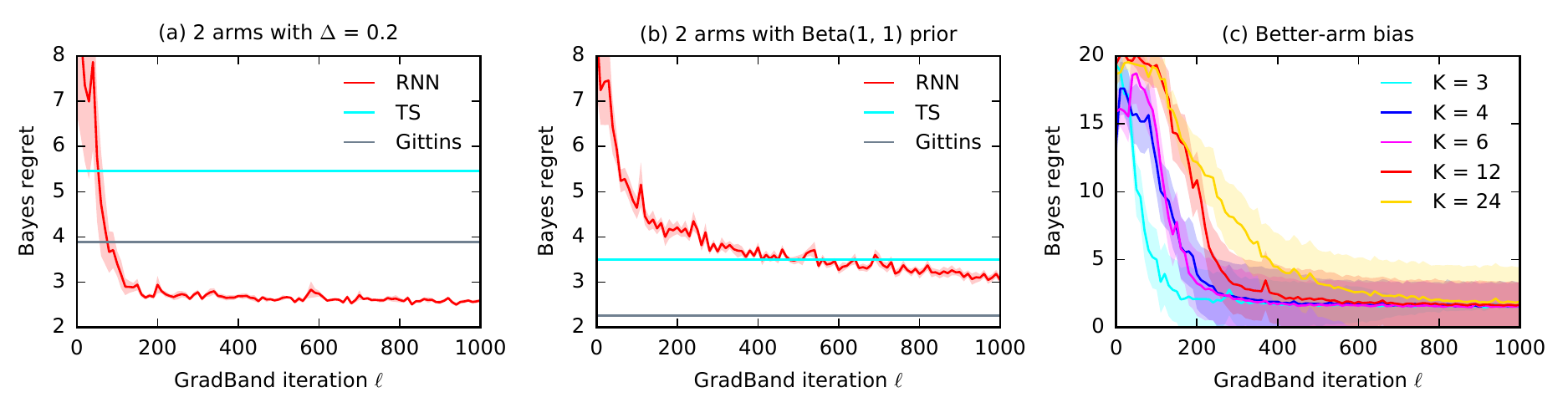} \\
  \vspace{-0.1in}
  \caption{The Bayes regret of RNN policies, as a function of \gradband iterations. We report the average over $10$ runs in the first two plots and the median in the last. The median excludes a few failed runs, which would skew the average.}
  \label{fig:rnn}
\end{figure*}

Now we optimize RNN policies (\cref{sec:rnn}) by \gradband. In our preliminary experiments (unreported), we observed that learning effective RNN policies over longer horizons, $n = 200$ rounds or longer, is challenging if we rely solely on our variance reduction baselines. Therefore, we propose the use of \emph{curriculum learning} \citep{bengio09curriculum} to further reduce the variance. The key idea is to apply \gradband successively to problems with \emph{increasing horizons}. In this experiment, we try a simple instance of this idea with two distinct horizons, $n' = 20$ and $n = 200$. First, we optimize the RNN policy using \gradband at the shorter horizon $n'$. Then we use this learned policy as the initial policy when we apply \gradband at the longer horizon $n$.


Results from the second optimization phase are reported in \cref{fig:rnn}. The number of \gradband iterations is $L = 1\,000$ and we made no attempt to optimize this scheme. In \cref{fig:rnn}a, we apply the RNN policy to the bandit problem in \cref{sec:policy optimization experiment}. The policy outperforms both \ts and the Gittins index. While unexpected, it does not contradict theory, since the Gittins index is not Bayes optimal for this problem. In \cref{fig:rnn}b, we consider a variant of this problem where arm means are drawn i.i.d.\ from $\mathrm{Beta}(1, 1)$. The Gittins index is Bayes optimal in this problem and thus outperforms the learned RNN policy. Nevertheless, the RNN policy still has a lower regret than \ts.

The last experiment is on a $K$-armed Bayesian bandit with Bernoulli rewards. We vary $K$ from $3$ to $24$. The prior distribution $\cP$ is a mixture of two problem instances,
\begin{align*}
  \cP(\theta_*)
  = \frac{1}{2} \I{\theta_* = (0.6, 0.9, 0.7, 0.7, \dots, 0.7)} +
  \frac{1}{2} \I{\theta_* = (0.2, 0.7, 0.9, 0.7, \dots, 0.7)}\,,
\end{align*}
each with prior probability $0.5$. This problem has an interesting structure. The optimal arm can be easily identified by pulling arm $1$. This arm has the largest difference in mean rewards between the problem instances and thus can be used to identify the instance. Arms $4$ through $K$ are \emph{distractors}, they are suboptimal and have identical mean rewards in both instances. Interestingly, our RNN policies do not learn to pull arm $1$ to identify the instance, but learn a different strategy specialized to this problem. The strategy pulls only arms $2$ or $3$, as these are the only arms that can be optimal. Thus the RNN policies learn to ignore the distractors and their Bayes regret does not increase with $K$ (\cref{fig:rnn}c), in contrast to the behavior of classical bandit algorithms.

\subsection{Contextual Policy Optimization}
\label{sec:contextual policy optimization experiment}

\begin{figure}
  \centering
  \vspace{-0.2in}
  \includegraphics[width=2.8in]{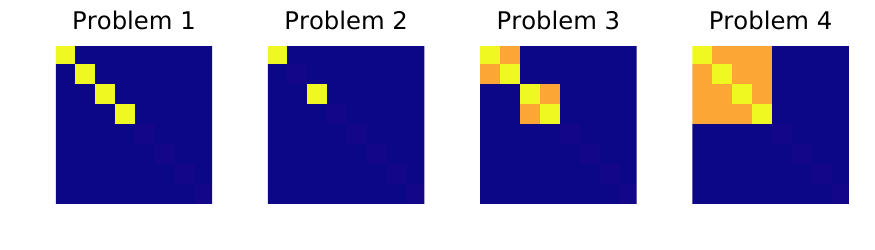} \\
  \vspace{-0.2in}
  \caption{Covariance matrices $\Sigma_\theta$ of synthetic problems in our contextual experiments. The yellow, orange, and navy colors represent $1$, $0.95$, and $0$, respectively.}
  \label{fig:problems}
\end{figure}

\begin{figure*}[t]
  \centering
  \includegraphics[width=6in]{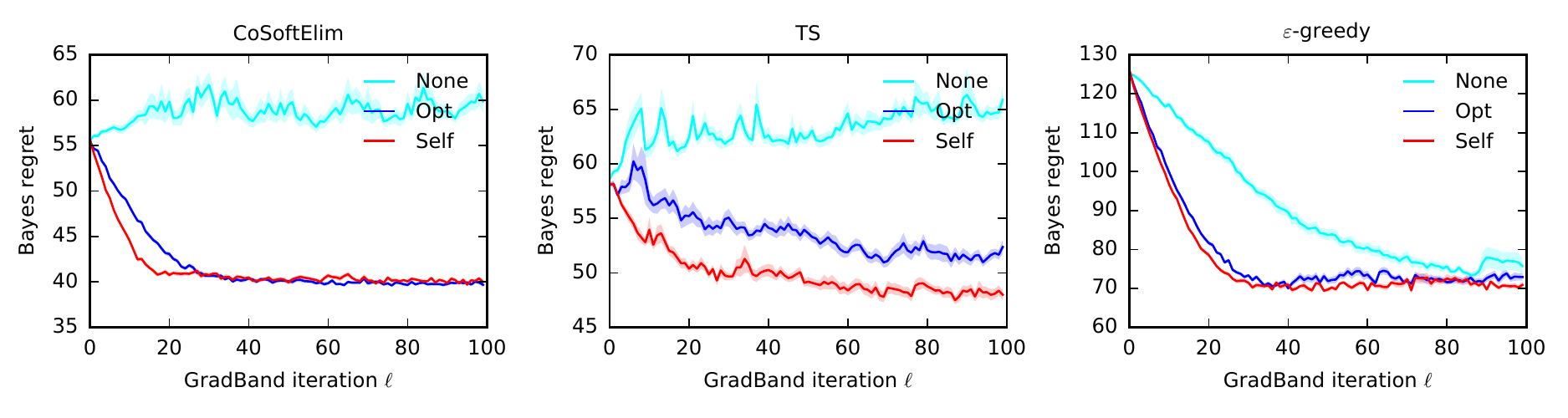} \\
  \vspace{-0.1in}
  \caption{The Bayes regret of \cosoftelim, \ts, and $\eps$-greedy policies with baselines $b^\textsc{none}$, $b^\textsc{opt}$, and $b^\textsc{self}$ in Problem 1 (\cref{sec:contextual policy optimization experiment}). The regret is averaged over $20$ runs.}
  \label{fig:baselines}
\end{figure*}

In the rest of this section, we apply \gradband to contextual bandit problems. In synthetic experiments, we use a set of Bayesian bandits dubbed Problems 1--4, which we describe below.

As in \cref{sec:policy optimization experiment}, we first show the benefit of baseline subtraction. We start with Problem 1, a contextual Bayesian bandit with $K = 4$ arms, $d = 8$ features, and horizon $n = 200$. The context in round $t$ is generated as $x_t \sim \cN(\mu_x, \Sigma_x)$, where $\mu_x = \mathbf{1}$ and $\Sigma_x = I_d$. The parameter vector of arm $i$ is $\theta_{i, *} \sim \cN(\mu_\theta, \Sigma_\theta)$, where $\mu_\theta = \mathbf{0}$ and its covariance matrix $\Sigma_\theta$ is visualized in \cref{fig:problems}. The reward of arm $i$ in round $t$ is $Y_{i, t} = x_t\T \theta_{i, *} + Z_{i, t}$, where $Z_{i, t} \sim \cN(0, \sigma^2)$ and $\sigma = 0.5$. In this problem, the parameter vectors $\theta_{i, *}$ lie in $4$ dimensions out of $8$. As a consequence, we expect \gradband to learn policies that ignore the other $4$ dimensions.

The optimization of \cosoftelim (\cref{sec:cosoftelim}) and \ts (\cref{sec:ts}) by \gradband is initialized using a projection matrix $W_0 = I_d$. The $\eps$-greedy policy (\cref{sec:e-greedy policy}) is parameterized by $\eps$ and initialized with $\eps_0 = 0.2$. The policies are optimized for $L = 100$ iterations with learning rate $\alpha = c^{-1} L^{- \frac{1}{2}}$ and batch size $m = 500$. As in \cref{sec:policy optimization experiment}, we set $c$ automatically so that $\norm{\hat{g}(n; \pi_{W_0})} \leq c$ holds with a high probability. This obviates the need for tuning the learning rate for each problem.

Our results are shown in \cref{fig:baselines}. We observe five trends. First, optimization of \cosoftelim and \ts by \gradband performs poorly without baseline subtraction, which shows the importance of the baselines in variance reduction. Second, the best optimized policy is \cosoftelim with $b^\textsc{self}$. Its regret decreases by $28\%$, from $55.74$ to $40.01$, and reaches a much lower point than any other tuned \ts or $\eps$-greedy policy. Third, optimization of $\ts$ is noisier than that of the other methods, due to the additional randomness in its reward gradient (\cref{lem:ts gradient}). This negatively impacts the quality of learned policies. Fourth, the $\eps$-greedy policy performs poorly because it optimizes only a single parameter. Finally, $b^\textsc{self}$ is the best baseline overall, across all three policy classes. Therefore, we use it in all remaining experiments.

\subsection{Relevant Subspace Meta-Learning}
\label{sec:subspace meta-learning experiment}

\begin{figure*}[t]
  \centering
  \includegraphics[width=6in]{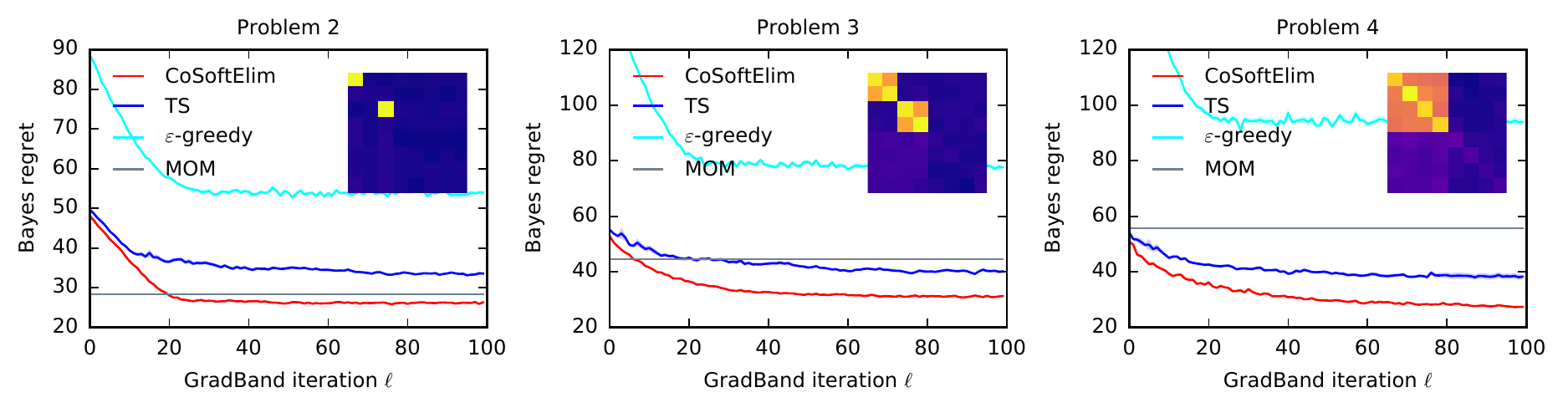} \\
  \vspace{-0.1in}
  \caption{The Bayes regret of \cosoftelim, \ts, and $\eps$-greedy policies on problems in \cref{sec:subspace meta-learning experiment}. The regret is averaged over $20$ runs. We also show the learned projection matrix $W$ in \cosoftelim. The gray line is \lints with the MOM estimated subspace in \cref{sec:subspace meta-learning experiment}.}
  \label{fig:subspace 8}
\end{figure*}

Now we show that \gradband can learn projection matrices that represent relevant subspaces, those containing $\theta_{i, *}$. We experiment with Problems 2--4, which differ from Problem 1 only in the configuration of the covariance matrix $\Sigma_\theta$, as shown in \cref{fig:problems}. In Problem 2, $\theta_{i, *}$ lie in dimensions $1$ and $3$. In Problem 3, $\theta_{i, *}$ lie in strongly correlated dimensions $1$--$2$ and $3$--$4$. In Problem 4, $\theta_{i, *}$ lie in $4$ strongly correlated dimensions $1$--$4$.

Our results are reported in \cref{fig:subspace 8}. The best optimized policy is \cosoftelim. After tuning by \gradband, its regret always decreases by roughly $50\%$. The $\eps$-greedy policy performs poorly, as its optimization is limited to a single parameter. \cref{fig:subspace 8} also shows the learned projections $W$ in \cosoftelim at $\ell = 100$. The projections resemble $\Sigma_\theta$ in \cref{fig:problems}, which indicates that \gradband learns the relevant subspace.

\cref{fig:subspace 8} also compares \cosoftelim to other bandit approaches. First, note that \ts at iteration $\ell = 1$ is vanilla \lints, as introduced in \cref{sec:ts} except that $W = I_d$. Both optimized \ts and \cosoftelim outperform it by a large margin. This shows the importance of learning suitable $W$.

The learning objective is critical as well. To show this, we implement \lints with a projection matrix $W$ that is estimated by a \emph{method of moments (MOM)} for linear meta-learning \citep{tripuraneni20provable}. The details of how $W$ is estimated are in \cref{sec:supplementary experiments}. \cosoftelim can outperform this approach in all problems in \cref{fig:subspace 8}. The reason is that efficient exploration is not only about relevant subspace recovery, but also about scaling $W$ in \eqref{eq:mle} with the goal of regret minimization.

\citet{cella20metalearning} recently proposed a \linucb-like method for meta-learning in linear bandits, where the learned parameter vector is biased towards the mean task vector. The mean task vector in our problems is $\mu_\theta = \mathbf{0}$. Under the assumption that $\mu_\theta$ is known, their method reduces to \linucb, which has regret of $116.02 \pm 0.78$, $99.01 \pm 0.69$, $134.57 \pm 0.98$, and $156.98 \pm 1.15$ in Problems 1--4, respectively. This is inferior to our proposed approaches.

\subsection{Multi-Class Classification}
\label{sec:multi-class classification experiment}

\begin{figure*}[t]
  \centering
  {\scriptsize
  \begin{tabular}{@{\ }l|r@{\ \ }r@{\ \ }r@{\ }} \hline
    Dataset & $K$ & $d$ & Examples \\ \hline
    Adult & $2$ & $123$ & $32\,561$ \\
    Australian Statlog & $2$ & $14$ & $690$ \\
    Breast Cancer & $2$ & $10$ & $683$ \\
    Covertype & $7$ & $54$ & $581\,012$ \\
    Iris & $3$ & $4$ & $150$ \\
    Image Segmentation & $7$ & $19$ & $2\,310$ \\
    Landsat Satellite Statlog & $6$ & $36$ & $4\,435$ \\
    Mushroom & $2$ & $112$ & $8\,124$ \\
    Shuttle Statlog & $7$ & $9$ & 43\,500 \\
    Vehicle Statlog & $4$ & $18$ & $846$ \\
    Wine & $3$ & $13$ & $178$ \\ \hline
  \end{tabular}
  }
  \raisebox{-0.775in}{
  \includegraphics[width=1.8in]{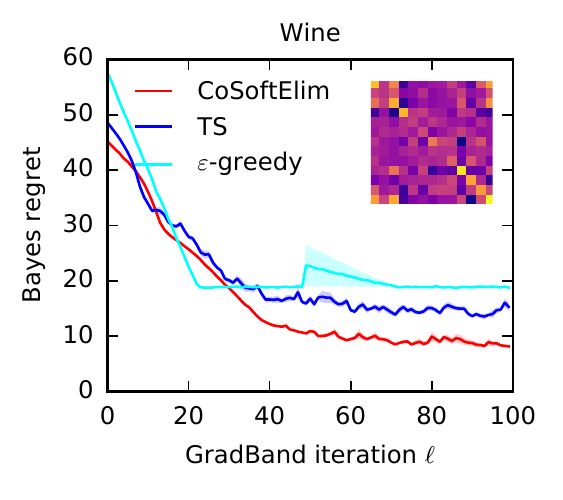}
  \includegraphics[width=1.8in]{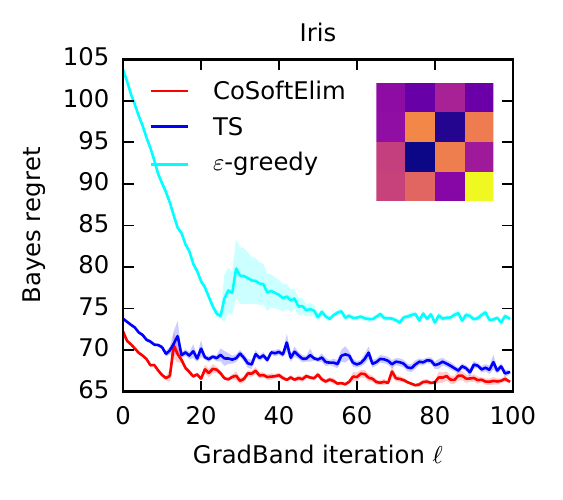}
  } \\
  \vspace{-0.05in}
  \caption{\textbf{Left.} Multi-class classification bandit problems in \cref{sec:multi-class classification experiment}. \textbf{Center.} The Bayes regret of \cosoftelim, \ts, and $\eps$-greedy policies on the Wine problem. The regret is averaged over $20$ runs. We also show the learned projection matrix $W$ in \cosoftelim. \textbf{Right.} An analogous plot for the Iris problem.}
  \label{fig:uci ml}
\end{figure*}

To demonstrate the generality of our approach, we apply it to multi-class classification bandit problems where arms correspond to labels of training examples \citep{agarwal14taming,riquelme18deep}. In round $t$, the agent observes feature vector $x_t$, which represents a training example in the classification problem, and then pulls an arm that represents the chosen label. The agent receives a reward of one if the pulled arm is the correct label, and zero otherwise. We experiment with $11$ datasets from the UCI ML Repository \citep{ucimlrepository}, comprising classification problems with up to $7$ classes and $123$ features. Properties of these datasets are listed in \cref{fig:uci ml}. The horizon is $n = 500$ rounds. When the number of features $d$ is smaller than $20$, the policies are optimized for $L = 100$ iterations with batch size $m = 1\,000$. Otherwise the policies are optimized for $L = 1\,000$ iterations with batch size $m = 80$. We increase $L$ since the solved problems become more challenging. We decrease $m$ to avoid memory overflow, as our current implementation stores $m n K$ sample covariance matrices in memory.

Detailed results for all problems are presented in \cref{sec:supplementary experiments}. Overall, we observe the same trends as in \cref{sec:subspace meta-learning experiment}. \cref{fig:uci ml} shows our best and worst results. In the Wine problem, optimization by \gradband reduces the regret of \cosoftelim more than fivefold, from $45.16$ to $8.18$. In the Iris problem, the regret of \cosoftelim is reduced by a mere $8\%$, from $72.21$ to $66.27$.

%% file: RelatedWork.tex
\section{Related Work}
\label{sec:related work}

Now we briefly discuss closely related work, either technically or in terms of its motivation. Similarly to our work, \citet{yang20differentiable} also extend \citet{boutilier20differentiable} to contextual bandits. The two extensions are different in multiple ways. For instance, we frame the problem of learning a contextual policy as a projection into a relevant subspace, whereas \citet{yang20differentiable} tune only the confidence interval width in a contextual extension of \softelim. In addition, we provide theoretical guarantees on policy-gradient optimization of contextual bandit policies in \cref{thm:contextual concavity}. Finally, we also propose a differentiable variant of Thompson sampling. Differentiable Thompson sampling was simultaneously proposed by \citet{min20policy}, but they do not consider its application to contextual problems.

Our problem is an instance of reinforcement learning (RL) \citep{sutton88learning} where the state is the history $H_t$ of the learning agent. The challenge is that the number of dimensions in $H_t$ is linear in the number of rounds $t$. Hence, in the absence of any additional structure, RL methods must deal with the \emph{curse of dimensionality}. Our approach is a policy-gradient method \citep{williams92simple,sutton00policy} with Monte-Carlo returns \citep{baxter01infinitehorizon}. Policy gradients typically have high variance when using Monte-Carlo methods, which gave rise to a number of variance reduction techniques based on baseline subtraction \citep{greensmith04variance,munos06geometric,zhao11analysis,dick15policy,liu18actiondependent}. Notably, \citet{munos06geometric} attain geometric variance reduction using sequential control variates while \citet{liu18actiondependent} propose control variates that also depend on the action. Our baselines (\cref{sec:baselines}) differ from typical baselines in RL. For instance, $b^\textsc{opt}$ relies on knowing the best arm in hindsight; and both $b^\textsc{opt}$ and $b^\textsc{self}$ use the fact that we can simulate all rewards in any problem instance $\theta_*$.

\citet{agarwal19optimality} and \citet{mei20global} recently proved asymptotic and finite bounds, respectively, on the quality of policy-gradient optimization of softmax policies. These bounds apply to problems with small state spaces, discounted rewards, or noise-free gradients. So, in their current form, they do not provide meaningful insight on \gradband, where the state space is exponential in the number of rounds, rewards are undiscounted, and gradients are noisy.

Our approach is a form of meta-learning \citep{thrun96explanationbased,thrun98lifelong}, where we learn from a sample of tasks to perform well across tasks from the same distribution \citep{baxter98theoretical,baxter00model}. Meta-learning has recently shown considerable promise \citep{finn17modelagnostic,finn18probabilistic,mishra18simple}. Sequential multitask learning \citep{caruana97multitask} is studied in multi-armed bandits by \citet{azar13sequential} and in contextual bandits by \citet{deshmukh17multitask}. In comparison, our setting is offline. A template for meta-learning of sequential strategies is developed by \citet{ortega19metalearning}. Meta-learning in linear models was recently analyzed \citep{bullins19generalize,tripuraneni20provable}. These approaches motivate our algorithms and we compare to them in \cref{sec:experiments}. \citet{cella20metalearning} propose a UCB algorithm for meta-learning in linear bandits but only model a simple bias, the mean task vector. This approach does not seem competitive with learning subspaces (\cref{sec:subspace meta-learning experiment}). Our setting is also more general, but we do not provide guarantees on meta-learning of \cosoftelim and \ts policies.

The regret of bandit algorithms can be reduced by tuning and this topic is addressed by many prior works \citep{vermorel05multiarmed,maes12metalearning,kuleshov14algorithms,hsu19empirical}. However, none of these works use policy gradients, neural network policies, or even the sequential structure of $n$-round rewards. \citet{duan16rl2} optimize a policy class similar to our RNN policies (\cref{sec:rnn}) by an existing optimizer. However, they do not formalize the objective clearly, relate their approach to Bayesian bandits, nor study policies that are provably sound (\cref{thm:concavity,thm:contextual concavity}). \citet{silver14deterministic} apply policy gradients to a continuous bandit problem with a quadratic cost function. Since their cost is convex in arms, this exploration problem is easier than with discrete arms.

The Bayes regret of classical bandit policies can be bounded \citep{russo14learning,wen15efficient,russo16information}. Since these policies also have instance-dependent regret bounds, they are more conservative than the policies that we adopt and tune in this work, where we directly minimize the Bayes regret.

\citet{maillard11thesis} propose \softelim with $w = 1$ and bound the number of pulls of suboptimal arms in Theorem 1.10. Their bound has a large $O(K \Delta^{-4})$ constant, which does not seem easy to eliminate. We introduce $w$ and have a tighter analysis (\cref{thm:softelim regret bound}) with a $O(1)$ constant. We also generalize \softelim to contextual bandits. \softelim resembles Boltzmann exploration \citep{sutton98reinforcement,cesabianchi17boltzmann} and \expthree. The key difference is in the design of $S_{i, t}$. In \expthree and Boltzmann exploration, $S_{i, t}$ only depends on the history of arm $i$. In \softelim, $S_{i, t}$ depends on all arms, which makes \softelim sufficiently optimistic.

%% file: Conclusions.tex
\section{Conclusions}
\label{sec:conclusions}

We take first steps towards understanding policy-gradient optimization of bandit policies. Our work addresses two main challenges that emerge in this problem. First, we derive the reward gradient of optimized policies that reflects the structure of our problem and show how to estimate it efficiently from samples. Second, we propose several differentiable bandit policies that, once meta-learned by \gradband, outperform state-of-the-art baselines. We study both the non-contextual and contextual setting. Our approach is general and works well in practice, as validated by extensive experiments.

Our goal was to demonstrate benefits of \emph{learning to explore} over state-of-the-art bandit policies. This is why we focus on two canonical classes of bandit problems, multi-armed and linear bandits. However, since our approach is general, we believe that it can be readily applied to other problem structures, such as non-linear generalized linear bandits \citep{filippi10parametric}, combinatorial actions in combinatorial semi-bandits \citep{gai12combinatorial,chen14combinatorial,kveton15combinatorial,wen15efficient}, and partial monitoring in online learning to rank \citep{radlinski08learning,kveton15cascading,lattimore18toprank}. We hope to extend our work to even more general structures, such as those studied by \citet{degenne20structure}, \citet{tirinzoni20novel}, and \citet{yu20graphical}.

We leave open several questions of interest. First, the variance of empirical reward gradients can be high, especially in RNN policies. So any progress in variance reduction would be of a great importance. Second, except for \cref{thm:concavity,thm:contextual concavity}, we are unaware of other bandit policy-instance pairs where the Bayes reward would be concave in the policy parameters, a property under which gradient ascent converges to optimal solutions. Our empirical observations (\cref{fig:gradients}a) suggest that this may be common. Finally, we believe that convergence guarantees for softmax exploration can be established based on recent advances in analyzing policy gradients in RL \citep{agarwal19optimality,russo19global,mei20global}, an important avenue for future research.

%% file: Concavity.tex
\section{Concave Bayes Reward}
\label{sec:concave bayes reward}

This appendix is organized as follows. In \cref{thm:concavity}, we show that the $n$-round Bayes reward of an explore-then-commit policy is concave in its exploration horizon. In \cref{thm:contextual concavity}, we generalize this results to the contextual setting.

\concavity*
\begin{proof}
We start with the \emph{explore-then-commit} policy \citep{langford08epochgreedy}, which is parameterized by $h \in [\floors{n / 2}]$ and works as follows. In the first $2 h$ rounds, it explores and pulls each arm $h$ times. Let $\hat{\mu}_{i, h}$ be the average reward of arm $i$ after $h$ pulls. Then, if $\hat{\mu}_{1, h} > \hat{\mu}_{2, h}$, arm $1$ is pulled for the remaining $n - 2 h$ rounds. Otherwise arm $2$ is pulled.

Fix any problem instance $\theta_* = (\mu_1, \mu_2)$. Without loss of generality, let arm $1$ be optimal, that is $\mu_1 > \mu_2$. Let $\Delta = \mu_1 - \mu_2$. The key observation is that the expected $n$-round reward in instance $\theta_*$ has a closed form
\begin{align}
  r(n, \theta_*; \pi_h)
  & = \mu_1 n - \Delta \left[h +
  \prob{\hat{\mu}_{1, h} < \hat{\mu}_{2, h}} (n - 2 h)\right]\,,
  \label{eq:concave reward}
\end{align}
where
\begin{align}
  \prob{\hat{\mu}_{1, h} < \hat{\mu}_{2, h}}
  & = \prob{\hat{\mu}_{1, h} - \hat{\mu}_{2, h} < 0}
  = \prob{\hat{\mu}_{1, h} - \hat{\mu}_{2, h} - \Delta < - \Delta}
  \nonumber \\
  & = \Phi\left(- \Delta \sqrt{h / 2}\right)
  = \frac{1}{\sqrt{2 \pi}}
  \int_{x = - \infty}^{- \Delta \sqrt{h / 2}} e^{- \frac{x^2}{2}} \dif x
  \label{eq:bad choice}
\end{align}
is the probability of committing to a suboptimal arm after the exploration phase. The third equality is from $\hat{\mu}_{1, h} - \hat{\mu}_{2, h} - \Delta \sim \cN(0, 2 / h)$, where $\Phi(x)$ denotes the cumulative distribution function of the standard normal distribution.

We want to prove that $r(n, \theta_*; \pi_h)$ is concave in $h$. We rely on the following property of convex functions of a single parameter $x$. Let $f(x)$ and $g(x)$ be non-negative, decreasing, and convex in $x$. Then $f(x) g(x)$ is non-negative, decreasing, and convex in $x$. This follows from
\begin{align*}
  (f(x) g(x))'
  & = f'(x) g(x) + f(x) g'(x)\,, \\
  (f(x) g(x))''
  & = f''(x) g(x) + 2 f'(x) g'(x) + f(x) g''(x)\,.
\end{align*}
It is easy to see that \eqref{eq:bad choice} is non-negative, decreasing, and convex in $h$. The same is true for $n - 2 h$, under our assumption that $h \in [\floors{n / 2}]$. As a result, $\prob{\hat{\mu}_{1, h} < \hat{\mu}_{2, h}} (n - 2 h)$ is convex in $h$, and so is $\Delta [h + \prob{\hat{\mu}_{1, h} < \hat{\mu}_{2, h}} (n - 2 h)]$. Therefore, $\eqref{eq:concave reward}$ is concave in $h$. Finally, the Bayes reward is concave in $h$ because $r(n; \pi_h) = \E{r(n,  \theta_*; \pi_h)}$.

The last remaining issue is that parameter $h$ in the explore-then-commit policy cannot be optimized by \gradband, as it is discrete. To allow its optimization, we extend the policy to continuous $h$ by randomized rounding.

The \emph{randomized explore-then-commit} policy has a continuous parameter $h \in [1, \floors{n / 2}]$. The discrete $\bar{h}$ is chosen as $\bar{h} = \floors{h} + Z$, where $Z \sim \mathrm{Ber}(h - \floors{h})$. Then we run the original policy with $\bar{h}$. The key property of the randomized policy is that its $n$-round Bayes reward is a piecewise linear interpolation of that of the original policy,
\begin{align*}
  (\ceils{h} - h) \, r(n; \pi_{\floors{h}}) +
  (h - \floors{h}) \, r(n; \pi_{\ceils{h}})\,.
\end{align*}
By definition, the above function is continuous and concave in $h$. This concludes the proof.
\end{proof}

\begin{algorithm}[t]
  \caption{Randomized contextual explore-then-commit policy.}
  \label{alg:contextual explore-then-commit}
  \begin{algorithmic}[1]
    \State \textbf{Inputs:} Continuous exploration horizon $h$
    \Statex
    \State $\bar{h} \gets \floors{h} + Z$, where $Z \sim \mathrm{Ber}(h - \floors{h})$
    \Comment{Randomized horizon rounding}
    \State $\forall i \in [2], j \in [L]: \hat{\mu}_{i, j} \gets 0$
    \Comment{Initialize estimated mean rewards of all arms}
    \For{$t = 1, \dots, n$}
      \State $j \gets x_t$
      \State $s \gets \sum_{\ell = 1}^{t - 1} \I{x_\ell = j}$
      \Comment{Number of past observations in context $j$}
      \If{$s \leq 2 \bar{h}$}
      \Comment{Explore}
        \If{$s$ is even}
          \State Pull arm $1$ and observe its reward $Y_{1, t}$
          \State $\hat{\mu}_{1, j} \gets (\hat{\mu}_{1, j} N + Y_{1, t}) / (N + 1)$
          where $N \gets s / 2$
        \Else
          \State Pull arm $2$ and observe its reward $Y_{2, t}$
          \State $\hat{\mu}_{2, j} \gets (\hat{\mu}_{2, j} N + Y_{2, t}) / (N + 1)$
          where $N \gets (s - 1) / 2$
        \EndIf
      \Else
      \Comment{Exploit}
        \State \algorithmicif\ $\hat{\mu}_{1, j} > \hat{\mu}_{2, j}$
        \algorithmicthen\ Pull arm $1$
        \algorithmicelse\ Pull arm $2$
      \EndIf
    \EndFor
  \end{algorithmic}
\end{algorithm}

\coconcavity*
\begin{proof}
The number of arms is $K = 2$ and the number of contexts is $L$. The key step in our proof is that the expected $n$-round reward is concave in $h$ for a carefully-chosen problem instance $\theta_*$. Then $r(n; h)$ is concave for any distribution over $\theta_*$. We define the problem instance as
\begin{align*}
  \theta_*
  = (\mu_{1, 1}, \mu_{2, 1}) \oplus \dots \oplus (\mu_{1, L}, \mu_{2, L})\,,
\end{align*}
where $\mu_{i, j}$ is the expected reward of arm $i$ in context $j$. The realized reward of arm $i$ in round $t$ is $Y_{i, t} \sim \cN(\mu_{i, x_t}, \sigma^2)$ for $\sigma = 1$.

The policy is a contextual variant of the randomized explore-then-commit policy (\cref{thm:concavity}) and we show it in \cref{alg:contextual explore-then-commit}. It is parameterized by a real-valued exploration horizon $h$, which is randomly rounded to the nearest integer $\bar{h}$. In each context $j$, each arm is explored $\bar{h}$ times. After that, the policy commits to the arm with the highest empirical mean in that context.

Let $S_j = \set{t \in [n]: x_t = j}$ be the rounds with context $j$ and $r_j(n, \theta_*; h)$ be the corresponding cumulative reward. Now note that the problem in any context $j$ is an instance of that in \cref{thm:concavity}. Therefore, $r_j(n, \theta_*; h)$ is concave in $h$ for any $h \in [1, \floors{\abs{S_j} / 2}]$. Moreover, since
\begin{align*}
  r(n, \theta_*; h)
  = \sum_{j = 1}^L r_j(n, \theta_*; h)\,,
\end{align*}
$r(n, \theta_*; h)$ is concave in $h$ for any $h \in [1, \min_{j \in [L]} \floors{\abs{S_j} / 2}]$; and so is $r(n; h) = \E{r(n, \theta_*; h)}$ for any distribution $\cP$ over $\theta_*$. This concludes the proof.
\end{proof}

%% file: GradientProofs.tex
\section{Gradient Proofs}
\label{sec:gradient proofs}

All proofs below are under the assumption that the sequence of contexts $x_{1 : n}$ is fixed (\cref{sec:bayesian bandits}). To simplify notation, do not explicitly condition on $x_{1 : n}$.

\gradient*
\begin{proof}
The $n$-round Bayes reward can be expressed as $r(n; \pi_w) = \E{\condE{\sum_{t = 1}^n Y_{I_t, t}}{Y}}$, where the outer expectation is over problem instances $\theta_*$ and their realized rewards $Y$, which do not depend on $w$. Thus
\begin{align*}
  \nabla_w r(n; \pi_w)
  = \E{\sum_{t = 1}^n \nabla_w \condE{Y_{I_t, t}}{Y}}\,.
\end{align*}
Only the pulled arms are random in the inner expectation. Therefore, for any $t \in [n]$, we have
\begin{align*}
  \condE{Y_{I_t, t}}{Y}
  = \sum_{i_{1 : t}} \condprob{I_{1 : t} = i_{1 : t}}{Y} Y_{i_t, t}\,.
\end{align*}
Now note that $\condprob{I_{1 : t} = i_{1 : t}}{Y}$ can be decomposed by the chain rule of probabilities as
\begin{align}
  \condprob{I_{1 : t} = i_{1 : t}}{Y}
  = \prod_{s = 1}^t \condprob{I_s = i_s}{I_{1 : s - 1} = i_{1 : s - 1}, Y}\,.
  \label{eq:chain rule}
\end{align}
Since the policy does not use $\theta_*$, future contexts, and future rewards, we have for any $s \in [n]$ that
\begin{align}
  \condprob{I_s = i_s}{I_{1 : s - 1} = i_{1 : s - 1}, Y}
  = \pi_w(i_s \mid i_{1 : s - 1}, x_1, \dots, x_s,
  Y_{i_1, 1}, \dots, Y_{i_{s - 1}, s - 1})\,.
  \label{eq:policy equivalence}
\end{align}
Finally, note that $\nabla_w f(\pi_w) = f(\pi_w) \nabla_{\pi_w} \log f(\pi_w)$ holds for any non-negative differentiable function $f$. This is known as the score-function identity \citep{aleksandrov68stochastic} and is the basis of policy-gradient methods. We apply it to $\condE{Y_{I_t, t}}{Y}$ and obtain
\begin{align*}
  \nabla_w \condE{Y_{I_t, t}}{Y}
  & = \sum_{i_{1 : t}} Y_{i_t, t} \nabla_w \condprob{I_{1 : t} = i_{1 : t}}{Y} \\
  & = \sum_{i_{1 : t}} Y_{i_t, t} \, \condprob{I_{1 : t} = i_{1 : t}}{Y}
  \nabla_w \log \condprob{I_{1 : t} = i_{1 : t}}{Y} \\
  & = \sum_{s = 1}^t \condE{Y_{I_t, t} \nabla_w \log \pi_w(I_s \mid H_s)}{Y}\,,
\end{align*}
where the last equality is by \eqref{eq:chain rule} and \eqref{eq:policy equivalence}. Now we chain all equalities and rearrange the result as
\begin{align*}
  \nabla_w r(n; \pi_w)
  = \sum_{t = 1}^n \sum_{s = 1}^t \E{Y_{I_t, t}
  \nabla_w \log \pi_w(I_s \mid H_s)}
  = \sum_{t = 1}^n \E{\nabla_w \log \pi_w(I_t \mid H_t)
  \sum_{s = t}^n Y_{I_s, s}}\,.
\end{align*}
This concludes the proof.
\end{proof}

\baselinegradient*
\begin{proof}
Fix round $t$. We want to show that $b_t$ does not change the expectation in \cref{lem:gradient}. That is,
\begin{align*}
  \E{b_t(I_{1 : t - 1}, Y, x_{1 : n}, \theta_*) \nabla_w \log \pi_w(I_t \mid H_t)}
  = 0\,.
\end{align*}
We proceed as follows. Since $b_t$ does not depend on $I_t$,
\begin{align*}
  & \E{b_t(I_{1 : t - 1}, Y, x_{1 : n}, \theta_*) \nabla_w \log \pi_w(I_t \mid H_t)} \\
  & \quad = \E{b_t(I_{1 : t - 1}, Y, x_{1 : n}, \theta_*)
  \condE{\nabla_w \log \pi_w(I_t \mid H_t)}{I_{1 : t - 1}, Y}}\,.
\end{align*}
Now note that
\begin{align*}
  \condE{\nabla_w \log \pi_w(I_t \mid H_t)}{I_{1 : t - 1}, Y}
  & = \sum_{i = 1}^K \condprob{I_t = i}{I_{1 : t - 1}, Y}
  \nabla_w \log \pi_w(i \mid H_t) \\
  & = \sum_{i = 1}^K \pi_w(i \mid H_t)
  \nabla_w \log \pi_w(i \mid H_t) \\
  & = \nabla_w \sum_{i = 1}^K \pi_w(i \mid H_t)\,.
\end{align*}
Since $\sum_{i = 1}^K \pi_w(i \mid H_t) = 1$, we have $\nabla_w \sum_{i = 1}^K \pi_w(i \mid H_t) = 0$. This concludes the proof.
\end{proof}

\baselineopt*
\begin{proof}
Let $R_s = Y_{I_{*, s}, s} - Y_{I_s, s}$ be the regret in round $s$ and $\bar{R}_s = f_{I_{*, s}}(x_s, \theta_*) - f_{I_s}(x_s, \theta_*)$ be its expectation, conditioned on the pulled arm and context. Then from the definition of $b^\textsc{opt}$, we have
\begin{align*}
  \nabla_w r(n; \pi_w)
  & = \E{\sum_{t = 1}^n \nabla_w \log \pi_w(I_t \mid H_t) \sum_{s = t}^n R_s}
  = \E{\sum_{t = 1}^n \nabla_w \log \pi_w(I_t \mid H_t) \sum_{s = t}^n \bar{R}_s} \\
  & = \E{\sum_{t = 1}^n \bar{R}_t \sum_{s = 1}^t \nabla_w \log \pi_w(I_s \mid H_s)}\,.
\end{align*}
Now we take the norm of the reward gradient and bound it from above as
\begin{align*}
  \normw{\nabla_w r(n; \pi_w)}{2}
  \leq \E{\sum_{t = 1}^n \abs{\bar{R}_t}
  \sum_{s = 1}^t \normw{\nabla_w \log \pi_w(I_s \mid H_s)}{2}}
  \leq c n \, \E{\sum_{t = 1}^n \abs{\bar{R}_t}}\,,
\end{align*}
where the first inequality is from the subadditivity of the Euclidean norm and the second is by the definition of $c$.

Let $\Delta_{\min}$ and $\Delta_{\max}$ be the minimum and maximum gaps, respectively, as defined in the claim. Then we have $\abs{\bar{R}_t} \leq \Delta_{\max} \I{I_t \neq I_{*, t}}$. Moreover, the regret in any problem instance is bounded from below as
\begin{align*}
  \E{R(n, \theta_*; \pi_w)}
  \geq \Delta_{\min} \, \E{\sum_{t = 1}^n \I{I_t \neq I_{*, t}}}\,.
\end{align*}
Now we combine these facts with the regret bound on any problem instance and get
\begin{align*}
  \normw{\nabla_w r(n; \pi_w)}{2}
  & \leq c \Delta_{\max} n \, \E{\sum_{t = 1}^n \I{I_t \neq I_{*, t}}}
  \leq c \frac{\Delta_{\max}}{\Delta_{\min}} n \, \E{R(n, \theta_*; \pi_w)} \\
  & = c \frac{\Delta_{\max}}{\Delta_{\min}} n O(n^\alpha)
  = O(n^{\alpha + 1})\,.
\end{align*}
This concludes the proof.
\end{proof}

\expthreederivative*
\begin{proof}
First, we express the derivative of $\log \pi_{i, t}$ with respect to $w$ as
\begin{align*}
  \nabla_w \log \pi_{i, t}
  = \frac{1}{\pi_{i, t}} \nabla_w \pi_{i, t}
  = \frac{1}{\pi_{i, t}}
  \left[(1 - w) \nabla_w \frac{V_{i, t}}{V_t} -
  \frac{V_{i, t}}{V_t} + \frac{1}{K}\right]\,.
\end{align*}
Conditioned on the history, $S_{i, t}$ is a constant independent of $w$, and thus we have
\begin{align*}
  \nabla_w \frac{V_{i, t}}{V_t}
  & = \frac{1}{V_t} \nabla_w V_{i, t} + V_{i, t} \nabla_w \frac{1}{V_t}
  = \frac{V_{i, t} S_{i, t}}{V_t K} -
  \frac{V_{i, t}}{V_t^2} \sum_{j = 1}^K V_{j, t} \frac{S_{j, t}}{K} \\
  & = \frac{V_{i, t}}{V_t} \left[\frac{S_{i, t}}{K} -
  \sum_{j = 1}^K \frac{V_{j, t}}{V_t} \frac{S_{j, t}}{K}\right]\,.
\end{align*}
This concludes the proof.
\end{proof}

\tsgradient*
\begin{proof}
The key idea is to rederive $\nabla_w \condE{Y_{I_t, t}}{Y}$ in \cref{lem:gradient}, with sampled posterior means in Thompson sampling. The remaining steps are the same as in \cref{lem:gradient,lem:baseline gradient}.

Let $\tilde{\mu}_t = (\tilde{\mu}_{1, t}, \dots, \tilde{\mu}_{K, t})$ be all posterior-sampled means in round $t$ and $\tilde{\mu}_{1 : t}$ be all posterior-sampled means in the first $t$ rounds. Then, analogously to the chain rule in \eqref{eq:chain rule}, we have
\begin{align*}
  \condprob{I_{1 : t} = i_{1 : t}}{Y}
  & = \prod_{s = 1}^t \condprob{I_s = i_s}{I_{1 : s - 1} = i_{1 : s - 1}, Y} \\
  & = \prod_{s = 1}^t \int_{\tilde{\mu}_s}
  \condprob{I_s = i_s, \tilde{\mu}_s} {I_{1 : s - 1} = i_{1 : s - 1}, Y}
  \dif \tilde{\mu}_s \\
  & = \prod_{s = 1}^t \int_{\tilde{\mu}_s}
  \condprob{I_s = i_s}{\tilde{\mu}_s}
  \condprob{\tilde{\mu}_s}{I_{1 : s - 1} = i_{1 : s - 1}, Y}
  \dif \tilde{\mu}_s \\
  & = \int_{\tilde{\mu}_{1 : t}}
  \underbrace{\left(\prod_{s = 1}^t
  \condprob{I_s = i_s}{\tilde{\mu}_s}\right)}_{\cI(\tilde{\mu}_{1 : t})} \
  \underbrace{\left(\prod_{s = 1}^t
  \condprob{\tilde{\mu}_s}{I_{1 : s - 1} = i_{1 : s - 1}, Y}\right)}_
  {\cM(\tilde{\mu}_{1 : t})}
  \dif \tilde{\mu}_{1 : t}\,.
\end{align*}
The third equality holds because $I_s$ depends only on $\tilde{\mu}_s$.

Since $\cI(\tilde{\mu}_{1 : t})$ is independent of $w$, we have
\begin{align*}
  \nabla_w \condprob{I_{1 : t} = i_{1 : t}}{Y}
  & = \nabla_w \int_{\tilde{\mu}_{1 : t}}
  \cI(\tilde{\mu}_{1 : t}) \, \cM(\tilde{\mu}_{1 : t})
  \dif \tilde{\mu}_{1 : t} \\
  & = \int_{\tilde{\mu}_{1 : t}}
  \cI(\tilde{\mu}_{1 : t}) \nabla_w \cM(\tilde{\mu}_{1 : t})
  \dif \tilde{\mu}_{1 : t} \\
  & = \int_{\tilde{\mu}_{1 : t}}
  \cI(\tilde{\mu}_{1 : t}) \, \cM(\tilde{\mu}_{1 : t})
  \nabla_w \log \cM(\tilde{\mu}_{1 : t})
  \dif \tilde{\mu}_{1 : t} \\
  & = \int_{\tilde{\mu}_{1 : t}}
  \condprob{I_{1 : t} = i_{1 : t}, \tilde{\mu}_{1 : t}}{Y}
  \nabla_w \log \cM(\tilde{\mu}_{1 : t})
  \dif \tilde{\mu}_{1 : t}\,.
\end{align*}
The third equality is by the score-function identity \citep{aleksandrov68stochastic}.

Finally, as in \cref{lem:gradient}, we have
\begin{align*}
  \nabla_w \condE{Y_{I_t, t}}{Y}
  & = \sum_{i_{1 : t}}
  Y_{i_t, t} \nabla_w \condprob{I_{1 : t} = i_{1 : t}}{Y} \\
  & = \sum_{i_{1 : t}} Y_{i_t, t} \int_{\tilde{\mu}_{1 : t}}
  \condprob{I_{1 : t} = i_{1 : t}, \tilde{\mu}_{1 : t}}{Y}
  \nabla_w \log \cM(\tilde{\mu}_{1 : t})
  \dif \tilde{\mu}_{1 : t} \\
  & = \condE{Y_{I_t, t} \nabla_w \log \cM(\tilde{\mu}_{1 : t})}{Y} \\
  & = \sum_{s = 1}^t \condE{Y_{I_t, t}
  \nabla_w \log \condprob{\tilde{\mu}_s}{I_{1 : s - 1}, Y}}{Y} \\
  & = \sum_{s = 1}^t \condE{Y_{I_t, t}
  \sum_{i = 1}^K \nabla_w \log p_i(\tilde{\mu}_{i, s} \mid H_s; w)}{Y}\,.
\end{align*}
The last equality follows from the fact that the posterior mean of each arm in round $s$ is sampled independently, and depends only on $H_s$ and $w$. Hence $\condprob{\tilde{\mu}_s}{I_{1 : s - 1}, Y} = \prod_{i = 1}^K p_i(\tilde{\mu}_{i, s} \mid H_s; w)$. This concludes the proof.
\end{proof}

%% file: Analysis.tex
\section{Analysis of \softelim}
\label{sec:softelim analysis}

We informally justify \softelim in \cref{sec:softelim sketch}. Then we bound its regret in \cref{sec:softelim regret bound}.

\subsection{Sketch}
\label{sec:softelim sketch}

We illustrate the main idea behind our analysis in an informal argument in a $2$-armed bandit. Let arm $1$ be optimal, that is $\mu_1 > \mu_2$. Let $\Delta = \mu_1 - \mu_2$ be the gap. Fix any round $t$ by which arm $2$ has been pulled \say{often,} so that we get $T_{2, t - 1} = \Omega(\Delta^{-2} \log n)$ and $\hat{\mu}_{2, t - 1} \leq \mu_2 + \Delta / 3$ with high probability. Let
\begin{align*}
  \hat{\mu}_{\max, t}
  = \max \set{\hat{\mu}_{1, t}, \hat{\mu}_{2, t}}\,.
\end{align*}
Now consider two cases. First, when $\hat{\mu}_{\max, t - 1} = \hat{\mu}_{1, t - 1}$, arm $1$ is pulled with probability at least $0.5$, by definition of $\pi_{1, t}$. Second, when $\hat{\mu}_{\max, t - 1} = \hat{\mu}_{2, t - 1}$, we have
\begin{align*}
  \pi_{1, t}
  = \exp[-2 (\hat{\mu}_{2, t - 1} - \hat{\mu}_{1, t - 1})^2 T_{1, t - 1}] \pi_{2, t}
  \geq \exp[-2 (\mu_1 - \hat{\mu}_{1, t - 1})^2 T_{1, t - 1}] \pi_{2, t}\,,
\end{align*}
where the last inequality follows from $\hat{\mu}_{1, t - 1} \leq \hat{\mu}_{2, t - 1} \leq \mu_2 + \Delta / 3 \leq \mu_1$, which holds with high probability. This means that arm $1$ is pulled \say{sufficiently often} relative to arm $2$, proportionally to how much $\hat{\mu}_{1, t - 1}$ deviates from $\mu_1$.

Therefore, \softelim eventually enters a regime where arm $1$ has been pulled \say{often,} so that $T_{1, t - 1} = \Omega(\Delta^{-2} \log n)$ and $\hat{\mu}_{1, t - 1} \geq \mu_1 - \Delta / 3$ holds with high probability. Then both $S_{1, t} = 0$ and $S_{2, t} = \Omega(\log n)$ hold with high probability, and the suboptimal arm $2$ is unlikely to be pulled.

\subsection{Regret Bound}
\label{sec:softelim regret bound}

We bound the $n$-round regret of \softelim below.

\softelimregretbound*
\begin{proof}
Each arm is initially pulled once. Therefore,
\begin{align*}
  R(n, \theta_*; \pi_w)
  = \sum_{i = 2}^K \Delta_i \left(\sum_{t = K + 1}^n \prob{I_t = i} + 1\right)\,.
\end{align*}
Now we decompose the probability of pulling each arm $i$ as
\begin{align*}
  \sum_{t = K + 1}^n \prob{I_t = i}
  = & \sum_{t = K + 1}^n \prob{I_t = i, T_{i, t - 1} \leq m} + {} \\
  & \sum_{t = K + 1}^n \prob{I_t = i, T_{i, t - 1} > m, T_{1, t - 1} \leq m} + {} \\
  & \sum_{t = K + 1}^n \prob{I_t = i, T_{i, t - 1} > m, T_{1, t - 1} > m}\,,
\end{align*}
where $m$ is chosen later. In the rest of the proof, we bound each above term separately. To simplify notation, use $\gamma = 1 / w^2$ in instead of $w^2$.

\subsection{Upper Bound on Term $1$}
\label{sec:softelim term 1}

Fix suboptimal arm $i$. Since $T_{i, t} = T_{i, t - 1} + 1$ on event $I_t = i$ and arm $i$ is initially pulled once, we have
\begin{align}
  \sum_{t = K + 1}^n \prob{I_t = i, T_{i, t - 1} \leq m}
  \leq m - 1\,.
  \label{eq:term 1}
\end{align}

\subsection{Upper Bound on Term $3$}
\label{sec:softelim term 3}

Fix suboptimal arm $i$ and round $t$. Let
\begin{align*}
  E_{1, t}
  = \set{\hat{\mu}_{1, t - 1} > \mu_1 - \frac{\Delta_i}{4}}\,, \quad
  E_{i, t}
  = \set{\hat{\mu}_{i, t - 1} < \mu_i + \frac{\Delta_i}{4}}\,,
\end{align*}
be the events that empirical means of arms $1$ and $i$, respectively, are \say{close} to their means. Then
\begin{align*}
  & \prob{I_t = i, T_{i, t - 1} > m, T_{1, t - 1} > m} \\
  & \quad \leq \prob{I_t = i, T_{i, t - 1} > m, E_{1, t}} +
  \prob{\bar{E}_{1, t}, T_{1, t - 1} > m} \\
  & \quad \leq \prob{I_t = i, T_{i, t - 1} > m, E_{1, t}, E_{i, t}} +
  \prob{\bar{E}_{1, t}, T_{1, t - 1} > m} +
  \prob{\bar{E}_{i, t}, T_{i, t - 1} > m}\,.
\end{align*}
Let $m = \ceils{16 \Delta_i^{-2} \log n}$. By the union bound and Hoeffding's inequality, we get
\begin{align*}
  \prob{\bar{E}_{1, t}, T_{1, t - 1} > m}
  & \leq \sum_{s = m + 1}^n
  \prob{\mu_1 - \hat{\mu}_{1, t - 1} \geq \frac{\Delta_i}{4}, \, T_{1, t - 1} = s}
  < n \exp\left[-2 \frac{\Delta_i^2}{16} m\right]
  = n^{-1}\,, \\
  \prob{\bar{E}_{i, t}, T_{i, t - 1} > m}
  & \leq \sum_{s = m + 1}^n
  \prob{\hat{\mu}_{i, t - 1} - \mu_i \geq \frac{\Delta_i}{4}, \, T_{i, t - 1} = s}
  < n \exp\left[-2 \frac{\Delta_i^2}{16} m\right]
  = n^{-1}\,.
\end{align*}
It follows that
\begin{align*}
  \prob{I_t = i, T_{i, t - 1} > m, T_{1, t - 1} > m}
  \leq \prob{I_t = i, T_{i, t - 1} > m, E_{1, t}, E_{i, t}} + 2 n^{-1}\,.
\end{align*}
Now note that $\hat{\mu}_{1, t - 1} - \hat{\mu}_{i, t - 1} \geq \Delta_i / 2$ on events $E_{1, t}$ and $E_{i, t}$. Let
\begin{align}
  \hat{\mu}_{\max, t - 1}
  = \max_{i \in [K]} \hat{\mu}_{i, t - 1}
  \label{eq:highest empirical mean}
\end{align}
be the highest empirical mean in round $t$. Since $\hat{\mu}_{\max, t - 1} \geq \hat{\mu}_{1, t - 1}$, we have $\hat{\mu}_{\max, t - 1} - \hat{\mu}_{i, t - 1} \geq \Delta_i / 2$. Therefore, on event $T_{i, t - 1} > m$, we get
\begin{align}
  p_{i, t}
  \leq \exp[-2 \gamma (\hat{\mu}_{\max, t - 1} - \hat{\mu}_{i, t - 1})^2 T_{i, t - 1}]
  \leq \exp\left[-2 \gamma \frac{\Delta_i^2}{4} m\right]
  \leq n^{-8 \gamma}\,.
  \label{eq:unlikely pull}
\end{align}
Finally, we chain all inequalities over all rounds and get that term $3$ is bounded as
\begin{align}
  \sum_{t = K + 1}^n \prob{I_t = i, T_{i, t - 1} > m, T_{1, t - 1} > m}
  \leq n^{1 - 8 \gamma} + 2\,.
  \label{eq:term 3}
\end{align}

\subsection{Upper Bound on Term $2$}
\label{sec:softelim term 2}

Fix suboptimal arm $i$ and round $t$. First, we apply Hoeffding's inequality to arm $i$, as in \cref{sec:softelim term 3}, and get
\begin{align*}
  \prob{I_t = i, T_{i, t - 1} > m, T_{1, t - 1} \leq m}
  & \leq \prob{I_t = i, T_{i, t - 1} > m, T_{1, t - 1} \leq m, E_{i, t}} + n^{-1} \\
  & = \E{p_{i, t} \I{T_{i, t - 1} > m, T_{1, t - 1} \leq m, E_{i, t}}} + n^{-1}\,.
\end{align*}
Let $\hat{\mu}_{\max, t - 1}$ be defined as in \eqref{eq:highest empirical mean}. Now we bound $p_{i, t}$ from above using $p_{1, t}$. We consider two cases. First, suppose that $\hat{\mu}_{\max, t - 1} > \mu_1 - \Delta_i / 4$. Then we have \eqref{eq:unlikely pull}. On the other hand, when $\hat{\mu}_{\max, t - 1} \leq \mu_1 - \Delta_i / 4$, we have
\begin{align}
  p_{i, t}
  = \frac{\exp[-2 \gamma (\hat{\mu}_{\max, t - 1} - \hat{\mu}_{i, t - 1})^2 T_{i, t - 1}]}
  {\exp[-2 \gamma (\hat{\mu}_{\max, t - 1} - \hat{\mu}_{1, t - 1})^2 T_{1, t - 1}]}
  p_{1, t}
  \leq \exp[2 \gamma (\mu_1 - \hat{\mu}_{1, t - 1})^2 T_{1, t - 1}]
  p_{1, t}\,.
  \label{eq:pi to p1}
\end{align}
It follows that
\begin{align*}
  p_{i, t}
  \leq \exp[2 \gamma (\mu_1 - \hat{\mu}_{1, t - 1})^2 T_{1, t - 1}] p_{1, t} +
  n^{-8 \gamma}\,,
\end{align*}
and we further get that
\begin{align*}
  & \E{p_{i, t} \I{T_{i, t - 1} > m, T_{1, t - 1} \leq m, E_{i, t}}} \\
  & \quad \leq \E{\exp[2 \gamma (\mu_1 - \hat{\mu}_{1, t - 1})^2
  T_{1, t - 1}] p_{1, t} \I{T_{1, t - 1} \leq m}} + n^{-8 \gamma} \\
  & \quad = \E{\exp[2 \gamma (\mu_1 - \hat{\mu}_{1, t - 1})^2
  T_{1, t - 1}] \I{I_t = 1, T_{1, t - 1} \leq m}} + n^{-8 \gamma}\,.
\end{align*}
With a slight abuse of notation, let $\hat{\mu}_{1, s}$ denote the average reward of arm $1$ after $s$ pulls. Then, since $T_{1, t} = T_{1, t - 1} + 1$ on event $I_t = 1$, we have
\begin{align*}
  \sum_{t = K + 1}^n \E{\exp[2 \gamma (\mu_1 - \hat{\mu}_{1, t - 1})^2
  T_{1, t - 1}] \I{I_t = 1, T_{1, t - 1} \leq m}}
  \leq \sum_{s = 1}^m \E{\exp[2 \gamma (\mu_1 - \hat{\mu}_{1, s})^2 s]}\,.
\end{align*}
Now fix the number of pulls $s$ and note that
\begin{align*}
  \E{\exp[2 \gamma (\mu_1 - \hat{\mu}_{1, s})^2 s]}
  & \leq \sum_{\ell = 0}^\infty
  \prob{\frac{\ell + 1}{\sqrt{s}} > \abs{\mu_1 - \hat{\mu}_{1, s}} \geq
  \frac{\ell}{\sqrt{s}}} \exp[2 \gamma (\ell + 1)^2] \\
  & \leq \sum_{\ell = 0}^\infty
  \prob{\abs{\mu_1 - \hat{\mu}_{1, s}} \geq \frac{\ell}{\sqrt{s}}}
  \exp[2 \gamma (\ell + 1)^2] \\
  & \leq 2 \sum_{\ell = 0}^\infty \exp[2 \gamma (\ell + 1)^2 - 2 \ell^2]\,,
\end{align*}
where the last step is by Hoeffding's inequality. The above sum can be easily bounded for any $\gamma < 1$. In particular, for $\gamma = 1 / 8$, the bound is
\begin{align*}
  \sum_{\ell = 0}^\infty \exp\left[\frac{(\ell + 1)^2}{4} - 2 \ell^2\right]
  \leq e^{\frac{1}{4}} + \sum_{\ell = 1}^\infty 2^{- \ell}
  \leq e\,.
\end{align*}
Now we combine all above inequalities and get that term $2$ is bounded as
\begin{align}
  \sum_{t = K + 1}^n \prob{I_t = i, T_{i, t - 1} > m, T_{1, t - 1} \leq m}
  \leq 2 e m + n^{1 - 8 \gamma} + 1\,.
  \label{eq:term 2}
\end{align}
Finally, we chain \eqref{eq:term 1}, \eqref{eq:term 3}, and \eqref{eq:term 2}; and use that $m \leq 16 \Delta_i^{-2} \log n + 1$.
\end{proof}

%% file: AnalysisCo.tex
\section{Analysis of \cosoftelim}
\label{sec:cosoftelim analysis}

We sketch the proof of \cosoftelim in \cref{sec:cosoftelim sketch}. Then we bound its regret in \cref{sec:cosoftelim regret bound}.

\subsection{Sketch}
\label{sec:cosoftelim sketch}

We rely on an equivalence of our problem and a linear bandit with $K d$ features. Therefore, we can build on two results from the analysis of \linucb \citep{abbasi-yadkori11improved}, the concentration of the MLE (\cref{lem:concentration} in \cref{sec:cosoftelim regret bound}) and that the sum of squared confidence widths of pulled arms is $\tilde{O}(K d)$ (\cref{lem:sum of squared confidence widths} in \cref{sec:cosoftelim regret bound}).

The last and most novel part of the analysis is an upper bound on the expected regret in round $t$ by the expected confidence widths of pulled arms. This bound is conditioned on history $H_t$ and relies heavily on the properties of softmax, in \eqref{eq:softmax} and \eqref{eq:cosoftelim score}. We argue along the following line. Let
\begin{align*}
  \Delta_{i, t}
  = \max_{j \in [K]} x_t\T \theta_{j, *} - x_t\T \theta_{i, *}
\end{align*}
be the gap of arm $i$ in round $t$. First, we show for any arm $i$ and \say{undersampled} arm $j$, an arm with a lot of uncertainty in direction $x_t$, that
\begin{align}
  \Delta_{i, t} \pi_{i, t}
  \leq 3 c_1 \sqrt{\log n} \normw{x_t}{G_{i, t - 1}^{-1}} \pi_{i, t} +
  12 c_1 \sqrt{\log n} \normw{x_t}{G_{j, t - 1}^{-1}} \pi_{i, t} +
  \Delta_{\max} n^{-1}\,.
  \label{eq:gap upper bound}
\end{align}
Roughly speaking, the bound is proved as follows. If arm $i$ is \say{undersampled}, its gap is bounded by its confidence width; and thus term $1$. If arm $i$ is \say{oversampled} and $\hat{\mu}_{j, t}$ is sufficiently high, arm $i$ is unlikely to be pulled; and thus term $3$. In all other cases, the gap of arm $i$ can be bounded by the confidence width of arm $j$; and thus term $2$. This is proved in \cref{lem:gap upper bound} below.

Second, we choose an appropriate \say{undersampled} arm $j$ to get an upper bound on the second $\pi_{i, t}$ in \eqref{eq:gap upper bound} using $\pi_{j, t}$. Finally, we sum up the upper bounds over all arms $i$ and get
\begin{align*}
  \Et{\Delta_{I_t, t}}
  \leq (12 K e + 3) c_1 \sqrt{\log n} \, \Et{\normw{x_t}{G_{I_t, t - 1}^{-1}}} +
  K \Delta_{\max} \delta\,.
\end{align*}
This is proved in \cref{lem:per-round regret} below.

\subsection{Regret Bound}
\label{sec:cosoftelim regret bound}

Our proof relies on equivalence between our problem and a linear bandit with $K d$ features, which we discuss next. Let $u_{i, t} \in \realset^{K d}$ be a context vector where $x_t$ is at entries $d (i - 1) + 1, \dots, d i$ and all remaining entries are zeros. Let the joint parameter vector be $\theta_* = \theta_{1, *} \oplus \dots \oplus \theta_{K, *} \in \realset^{K d}$ and the joint estimated vector be $\hat{\theta}_t = \hat{\theta}_{1, t} \oplus \dots \oplus \hat{\theta}_{K, t} \in \realset^{K d}$. Let $M_t \in \realset^{K d \times K d}$ be a block-diagonal matrix with blocks $G_{1, t}, \dots, G_{K, t}$. Then, for any arm $i$ in round $t$,
\begin{align*}
  u_{i, t}\T \theta_*
  = x_t\T \theta_{i, *}\,, \quad
  u_{i, t}\T \hat{\theta}_{t - 1}
  = x_t\T \hat{\theta}_{i, t - 1}\,, \quad
  \normw{u_{i, t}}{M_{t - 1}^{-1}}
  = \normw{x_t}{G_{i, t - 1}^{-1}}\,.
\end{align*}
The equivalence is useful because it allows us to reuse two existing results from the analysis of \linucb \citep{abbasi-yadkori11improved}, the concentration of the MLE (\cref{lem:concentration} in \cref{sec:cosoftelim regret bound}) and that the sum of squared confidence widths of pulled arms is $\tilde{O}(K d)$ (\cref{lem:sum of squared confidence widths} in \cref{sec:cosoftelim regret bound}).

The concentration part is solved as follows. Let
\begin{align}
  E_{1, t}
  = \set{\forall i \in [K]:
  |x_t\T \hat{\theta}_{i, t - 1} - x_t\T \theta_{i, *}| \leq
  c_1 \normw{x_t}{G_{i, t - 1}^{-1}}}
  \label{eq:concentration}
\end{align}
be the event that all estimated arm means in round $t$ are \say{close} to their actual means. Let $E_1 = \bigcap_{t = 1}^n E_{1, t}$ and $\bar{E}_1$ be its complement. The next lemma shows how to choose $c_1$ in \eqref{eq:concentration} such that event $\bar{E}_1$ is unlikely.

\begin{lemma}
\label{lem:concentration} For any $\sigma, \lambda, \delta > 0$, and
\begin{align*}
  c_1
  = \sigma \sqrt{K d \log \left(\frac{1 + n L^2 / (K d \lambda)}{\delta}\right)} +
  \lambda^\frac{1}{2} L_*\,,
\end{align*}
event $E_1$ occurs with probability at least $1 - \delta$.
\end{lemma}
\begin{proof}
Fix arm $i$ and round $t$. By the Cauchy-Schwarz inequality,
\begin{align*}
  u_{i, t}\T \hat{\theta}_{t - 1} - u_{i, t}\T \theta_*
  = u_{i, t}\T M_{t - 1}^{- \frac{1}{2}}
  M_{t - 1}^\frac{1}{2} (\hat{\theta}_{t - 1} - \theta_*)
  \leq \|\hat{\theta}_{t - 1} - \theta_*\|_{M_{t - 1}}
  \normw{u_{i, t}}{M_{t - 1}^{-1}}\,.
\end{align*}
By Theorem 2 of \citet{abbasi-yadkori11improved}, $\|\hat{\theta}_{t - 1} - \theta_*\|_{M_{t - 1}} \leq c_1$ holds jointly in all rounds $t \in [n]$ with probability of at least $1 - \delta$. This concludes the proof.
\end{proof}

We also use Lemma 11 of \citet{abbasi-yadkori11improved}, which bounds the sum of squared confidence widths of pulled arms.

\begin{lemma}
\label{lem:sum of squared confidence widths} For any $\lambda \geq L^2$, we have
$\displaystyle \sum_{t = 1}^n \normw{x_t}{G_{I_t, t - 1}^{-1}}^2 \leq c_2 = 2 K d \log(1 + n L^2 / (K d \lambda))$.
\end{lemma}

Let $\Delta_{i, t} = \max_{j \in [K]} x_t\T \theta_{j, *} - x_t\T \theta_{i, *}$ be the gap of arm $i$ in round $t$. Now we are ready to prove our main result.

\cosoftelimregretbound*
\begin{proof}
First, we split the $n$-round regret by event $E_1$ and apply \cref{lem:concentration} with $\delta = 1 / n$,
\begin{align*}
  R(n, P)
  & = \sum_{t = 1}^n \E{\Delta_{I_t, t}}
  \leq \sum_{t = 1}^n \E{\Delta_{I_t, t} \I{E_{1, t}}} +
  n \Delta_{\max} \prob{\bar{E}_1} \\
  & \leq \sum_{t = 1}^n \E{\Et{\Delta_{I_t, t}} \I{E_{1, t}}} + \Delta_{\max}\,.
\end{align*}
Second, we bound $\Et{\Delta_{I_t, t}} \I{E_{1, t}}$ from above using \cref{lem:per-round regret} with $\delta = 1 / n$ and get
\begin{align*}
  R(n, P)
  \leq (12 K e + 3) c_1 \sqrt{\log n} \,
  \E{\sum_{t = 1}^n \normw{x_t}{G_{I_t, t - 1}^{-1}}} + (K + 1) \Delta_{\max}\,.
\end{align*}
By the Cauchy-Schwarz inequality and \cref{lem:sum of squared confidence widths},
\begin{align*}
  \sum_{t = 1}^n \normw{x_t}{G_{I_t, t - 1}^{-1}}
  \leq \sqrt{n \sum_{t = 1}^n \normw{x_t}{G_{I_t, t - 1}^{-1}}^2}
  \leq \sqrt{c_2 n}\,.
\end{align*}
This concludes the proof.
\end{proof}

Our key lemmas are stated and proved below. We denote the mean reward of arm $i$ in round $t$ by $\mu_{i, t} = x_t\T \theta_{i, *}$. Let $i_{*, t} = \argmax_{i \in [K]} \mu_{i, t}$ be the optimal arm in round $t$ and $\mu_{*, t}$ be its mean reward. The key concepts in our analysis are undersampled and oversampled arms. We say that arm $i$ is \emph{undersampled} in round $t$ when $\displaystyle c_1 \normw{x_t}{G_{i, t - 1}^{-1}} \geq \frac{\Delta_{i, t}}{3 \sqrt{\log(1 / \delta)}}$. Otherwise the arm is \emph{oversampled}. When $\log(1 / \delta) \geq 1$, an oversampled arm $i$ satisfies $|\hat{\mu}_{i, t} - \mu_{i, t}| < \Delta_{i, t} / 3$ on event $E_{1, t}$. To simplify notation, we drop subindexing by $t$ in the proofs of the lemmas.

\begin{lemma}
\label{lem:gap upper bound} Fix history $H_t$ and assume that event $E_{1, t}$ occurs. Let $\gamma = c_1^{-2}$ and $\delta \in (0, 1]$ be chosen such that $\log(1 / \delta) \geq 1$. Then for any arm $i$ and undersampled arm $j$,
\begin{align*}
  \Delta_{i, t} p_{i, t}
  \leq 3 c_1 \sqrt{\log(1 / \delta)} \normw{x_t}{G_{i, t - 1}^{-1}} p_{i, t} +
  12 c_1 \sqrt{\log(1 / \delta)} \normw{x_t}{G_{j, t - 1}^{-1}} p_{i, t} +
  \Delta_{\max} \delta\,.
\end{align*}
\end{lemma}
\begin{proof}
We consider four cases. Case $1$ is that arm $i$ is undersampled. Then trivially
\begin{align*}
  \Delta_i
  \leq 3 c_1 \sqrt{\log(1 / \delta)} \normw{x}{G_i^{-1}}\,.
\end{align*}
Case $2$ is that arm $i$ is oversampled and the gap of arm $j$ is \say{large}, $\Delta_j \geq \Delta_i / 4$. Then
\begin{align*}
  \Delta_i
  \leq 4 \Delta_j
  \leq 12 c_1 \sqrt{\log(1 / \delta)} \normw{x}{G_j^{-1}}\,.
\end{align*}
Case $3$ is that arm $i$ is oversampled; the gap of arm $j$ is \say{small}, $\Delta_j < \Delta_i / 4$; and $\hat{\mu}_j \geq \mu_* - \Delta_i / 3$. In this case, arm $i$ is unlikely to be pulled for $\gamma = c_1^{-2}$,
\begin{align*}
  p_i
  \leq \exp\left[- \gamma \frac{(\hat{\mu}_{\max} - \hat{\mu}_i)^2}
  {\normw{x}{G_i^{-1}}^2}\right]
  \leq \exp\left[- \gamma \frac{(\hat{\mu}_j - \hat{\mu}_i)^2}
  {\normw{x}{G_i^{-1}}^2}\right]
  \leq \exp\left[- \gamma \frac{\Delta_i^2}{9}
  \frac{9 c_1^2 \log(1 / \delta)}{\Delta_i^2}\right]
  = \delta\,.
\end{align*}
The first inequality holds because the denominator in $p_i$ is at least $1$. The second inequality follows from $\hat{\mu}_{\max} - \hat{\mu}_i \geq \hat{\mu}_j - \hat{\mu}_i$, which holds from our assumption on $\hat{\mu}_j$ and that arm $i$ is oversampled. The last inequality follows from $\hat{\mu}_j - \hat{\mu}_i \geq \Delta_i / 3$ and that arm $i$ is oversampled. Since $\Delta_i \leq \Delta_{\max}$, we have that $\Delta_i p_i \leq \Delta_{\max} \delta$.

Case $4$ is that arm $i$ is oversampled, $\Delta_j < \Delta_i / 4$, and $\hat{\mu}_j < \mu_* - \Delta_i / 3$. Then
\begin{align*}
  \Delta_i
  \leq 3 (\mu_* - \hat{\mu}_j)
  = 3 (\mu_* - \mu_j + \mu_j - \hat{\mu}_j)
  \leq 3 (\Delta_j + c_1 \normw{x}{G_j^{-1}})
  \leq 12 c_1 \sqrt{\log(1 / \delta)} \normw{x}{G_j^{-1}}\,.
\end{align*}
Now we combine all four cases and get our claim.
\end{proof}

The above lemma is critical to prove \cref{lem:per-round regret} below, which bounds the expected regret in round $t$ by the expected confidence widths of pulled arms.

\begin{lemma}
\label{lem:per-round regret} Fix history $H_t$ and assume that event $E_{1, t}$ occurs. Let $\gamma = c_1^{-2}$ and $\delta \in (0, 1]$ be chosen such that $\log(1 / \delta) \geq 1$. Then
\begin{align*}
  \Et{\Delta_{I_t, t}}
  \leq (12 K e + 3) c_1 \sqrt{\log(1 / \delta)} \, \Et{\normw{x_t}{G_{I_t, t - 1}^{-1}}} +
  K \Delta_{\max} \delta\,.
\end{align*}
\end{lemma}
\begin{proof}
The proof has two parts. First, we bound $p_i$ from above using pulled undersampled arms. We consider two cases. Case $1$ is that the best empirical arm $i_{\max} = \argmax_{i \in [K]} \hat{\mu}_i$ is undersampled. Since $i_{\max}$ has the highest empirical mean, $p_{i_{\max}} \geq 1 / K$ and we have $p_i \leq K p_i p_j$ for $j = i_{\max}$.

Case $2$ is that arm $i_{\max}$ is oversampled. Because of that, $\hat{\mu}_{\max} \leq \mu_*$. Since the optimal arm $i_*$ is undersampled by definition, we have for $j = i_*$ that
\begin{align*}
  p_i
  = \frac{1}{p_j} p_i p_j
  \leq K \exp\left[\gamma \frac{(\hat{\mu}_{\max} - \hat{\mu}_j)^2}
  {\normw{x}{G_j^{-1}}^2}\right] p_i p_j
  \leq K \exp\left[\gamma \frac{(\mu_j - \hat{\mu}_j)^2}
  {\normw{x}{G_j^{-1}}^2}\right] p_i p_j
  \leq K e p_i p_j\,.
\end{align*}
By \cref{lem:gap upper bound} and from above, there exists an undersampled arm $j$ such that for any arm $i$,
\begin{align*}
  \Delta_i p_i
  \leq 3 c_1 \sqrt{\log(1 / \delta)} \normw{x}{G_i^{-1}} p_i +
  12 K e c_1 \sqrt{\log(1 / \delta)} \normw{x}{G_j^{-1}} p_i p_j +
  \Delta_{\max} \delta\,.
\end{align*}
Finally, we sum over all arms $i$ and get
\begin{align*}
  \Et{\Delta_{I_t}}
  & = \sum_{i = 1}^K \Delta_i p_i \\
  & \leq 3 c_1 \sqrt{\log(1 / \delta)}
  \left(\sum_{i = 1}^K \normw{x}{G_i^{-1}} p_i\right) +
  12 K e c_1 \sqrt{\log(1 / \delta)} \normw{x}{G_j^{-1}}
  \sum_{i = 1}^K p_i p_j +
  K \Delta_{\max} \delta \\
  & \leq 3 c_1 \sqrt{\log(1 / \delta)}
  \left(\sum_{i = 1}^K \normw{x}{G_i^{-1}} p_i\right) +
  12 K e c_1 \sqrt{\log(1 / \delta)} \normw{x}{G_j^{-1}} p_j +
  K \Delta_{\max} \delta \\
  & \leq (12 K e + 3) c_1 \sqrt{\log(1 / \delta)}
  \left(\sum_{i = 1}^K \normw{x}{G_i^{-1}} p_i\right) +
  K \Delta_{\max} \delta \\
  & = (12 K e + 3) c_1 \sqrt{\log(1 / \delta)} \, \Et{\normw{x}{G_{I_t}^{-1}}} +
  K \Delta_{\max} \delta\,.
\end{align*}
This concludes the proof.
\end{proof}

%% file: Experiments2.tex
\section{Supplementary Experiments}
\label{sec:supplementary experiments}

\subsection{Subspace Recovery}

\begin{figure*}[t]
  \centering
  \includegraphics[width=1.35in]{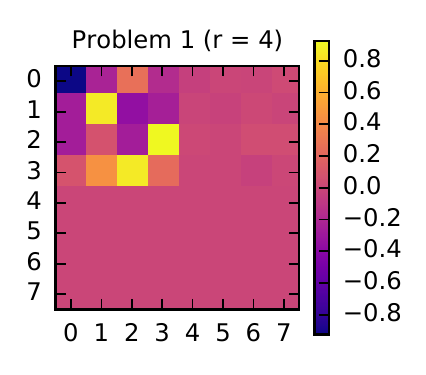}
  \includegraphics[width=1.35in]{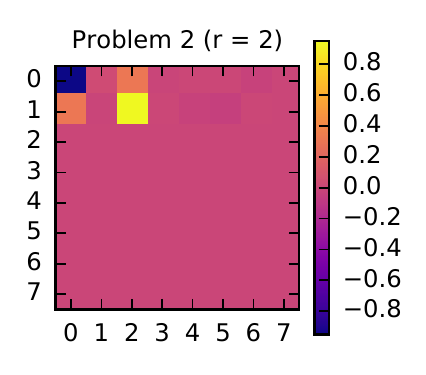}
  \includegraphics[width=1.35in]{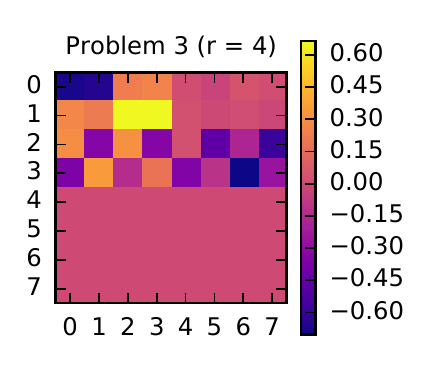}
  \includegraphics[width=1.35in]{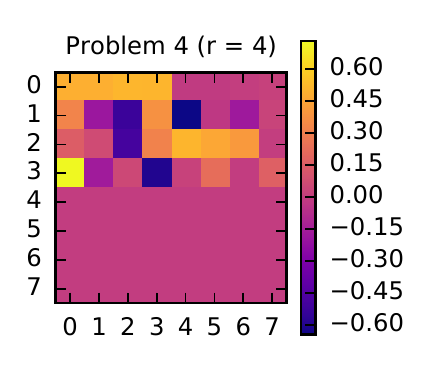} \\
  \vspace{-0.1in}
  \caption{Projection matrices $W$ estimated by MOM.}
  \label{fig:mom projections}
\end{figure*}

We use the method-of-moments (MOM) estimator for meta learning in linear models (Algorithm 1 of \citet{tripuraneni20provable}) to learn projection $W$. In particular, we sample $n = 100\, 000$ i.i.d.\ pairs $(x_t, Y_t)$ as
\begin{align*}
  x_t
  \sim \mathcal{N}(\mu_x, \Sigma_x)\,, \quad
  \theta_*
  \sim \mathcal{N}(\mu_\theta, \Sigma_\theta)\,, \quad
  Z_t
  \sim \mathcal{N}(0, \sigma^2)\,, \quad
  Y_t
  = x_t\T \theta_* + Z_t\,;
\end{align*}
and then estimate the subspace by applying PCA to $\sum_{t = 1}^n Y_t^2 x_t x_t\T$. The learned projection matrices $W$, together with the dimensionality of subspace $r$, are reported in \cref{fig:mom projections}. We hand-tuned $r$ to get good empirical performance in \cref{sec:subspace meta-learning experiment} and report these values of $r$ in \cref{fig:mom projections}. Note that \cosoftelim does not require such tuning.

\subsection{Real-World Experiments}

\begin{figure*}[t]
  \centering
  \includegraphics[width=1.8in]{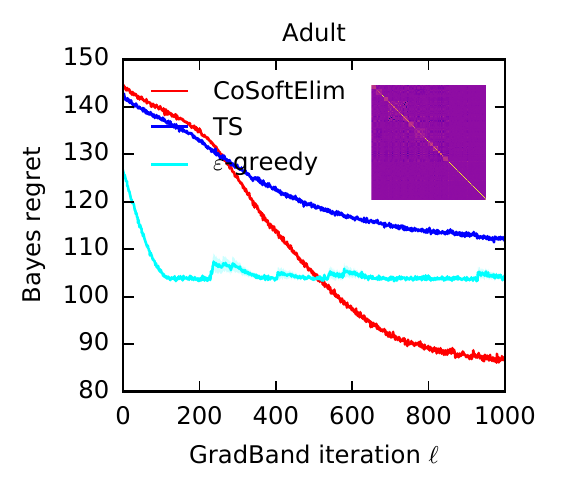}
  \includegraphics[width=1.8in]{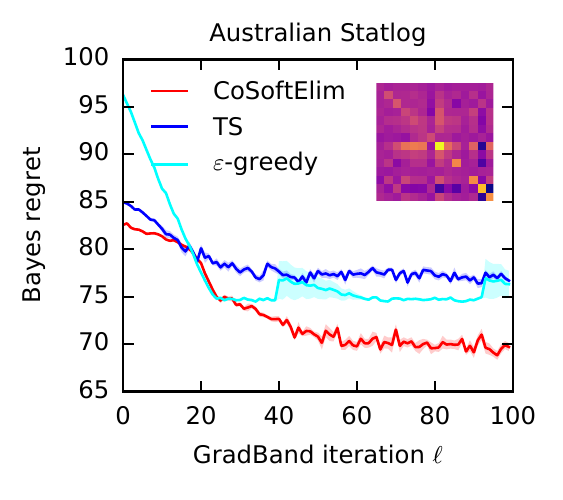}
  \includegraphics[width=1.8in]{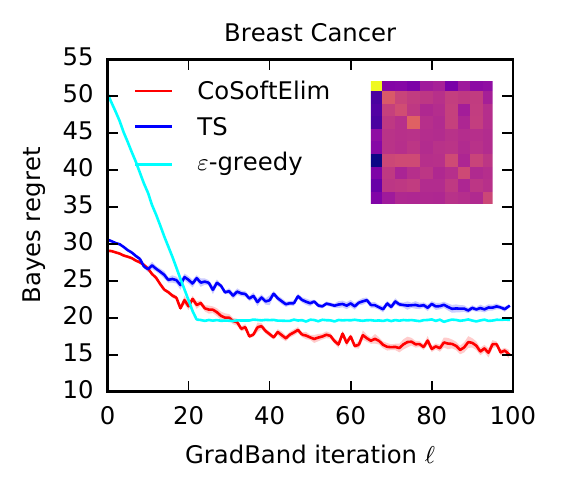}
  \includegraphics[width=1.8in]{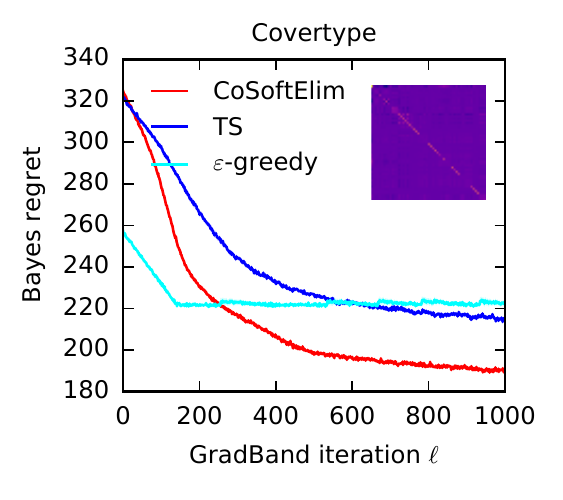}
  \includegraphics[width=1.8in]{Figures/Iris.pdf}
  \includegraphics[width=1.8in]{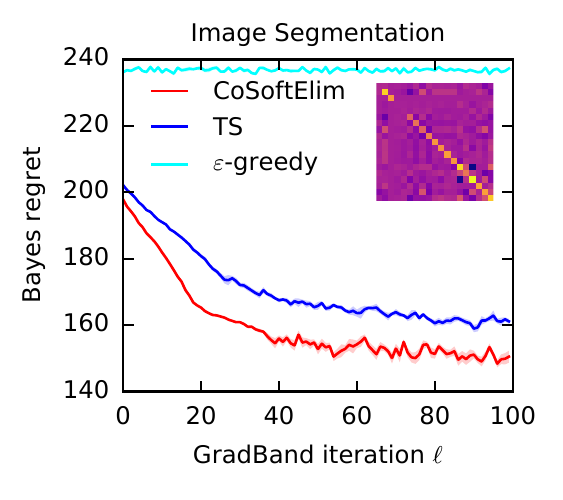}
  \includegraphics[width=1.8in]{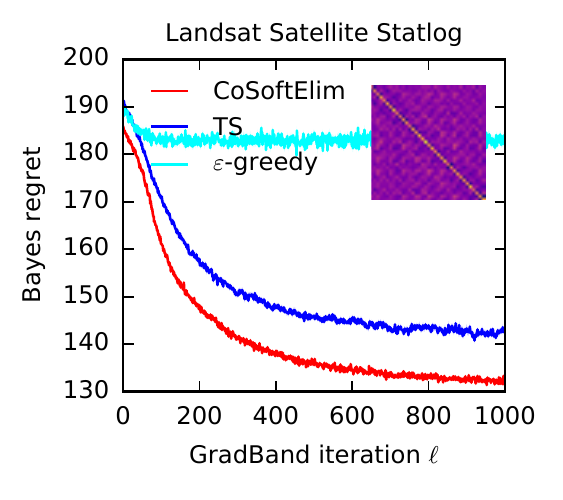}
  \includegraphics[width=1.8in]{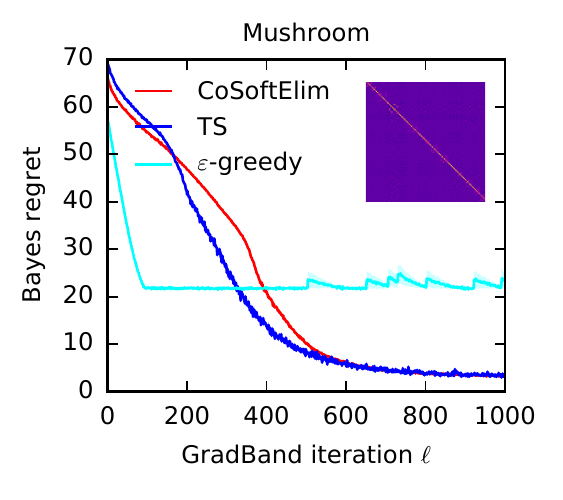}
  \includegraphics[width=1.8in]{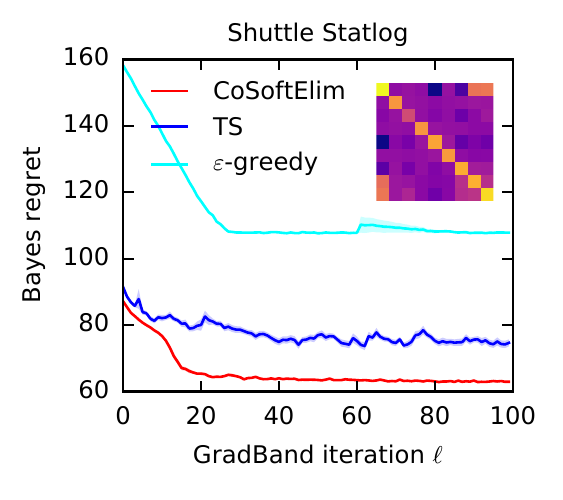}
  \includegraphics[width=1.8in]{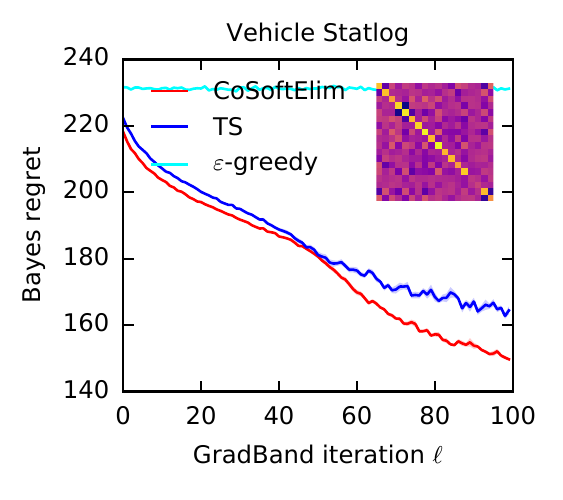}
  \includegraphics[width=1.8in]{Figures/Wine.pdf} \\
  \vspace{-0.1in}
  \caption{The Bayes regret of \cosoftelim, \ts, and $\eps$-greedy policies on all UCI ML repository problems in \cref{sec:multi-class classification experiment}. The regret is averaged over $20$ runs. We also show the learned projection matrix $W$ in \cosoftelim.}
  \label{fig:uci ml all}
\end{figure*}

Results for all UCI ML Repository datasets in \cref{sec:multi-class classification experiment} are reported in \cref{fig:uci ml all}. We observe significant improvements due to optimizing \cosoftelim. In particular, the regret decreases as
\begin{itemize}
  \item Adult: From $144.26 \pm 0.18$ to $86.90 \pm 0.35$, by $39.77\%$.
  \item Australian Statlog: From $82.58 \pm 0.11$ to $69.71 \pm 0.42$, by $16\%$.
  \item Breast Cancer: From $29.08 \pm 0.05$ to $15.10 \pm 0.57$, by $48\%$.
  \item Covertype: From $324.65 \pm 0.45$ to $190.28 \pm 0.46$m by $41.39\%$.
  \item Iris: From $72.21 \pm 0.09$ to $66.27 \pm 0.16$, by $8\%$.
  \item Image Segmentation: From $197.82 \pm 0.19$ to $150.53 \pm 1.81$, by $24\%$.
  \item Landsat Satellite Statlog: From $185.78 \pm 0.29$ to $132.24 \pm 0.41$, by $28.82\%$.
  \item Mushroom: From $66.25 \pm 0.11$ to $3.26 \pm 0.06$, by $95.07\%$.
  \item Shuttle Statlog: From $87.17 \pm 0.11$ to $62.99 \pm 0.15$, by $27.74\%$.
  \item Vehicle Statlog: From $218.16 \pm 0.22$ to $149.75 \pm 0.38$, by $31\%$.
  \item Wine: From $45.16 \pm 0.05$ to $8.18 \pm 0.24$, by $82\%$.
\end{itemize}

%% file: RNNDetails.tex
\definecolor{codegreen}{rgb}{0,0.6,0}
\definecolor{codegray}{rgb}{0.5,0.5,0.5}
\definecolor{codepurple}{rgb}{0.58,0,0.82}
\definecolor{backcolour}{rgb}{0.95,0.95,0.92}

\lstdefinestyle{mystyle}{
  backgroundcolor=\color{backcolour},
  commentstyle=\color{codegreen},
  keywordstyle=\color{magenta},
  numberstyle=\tiny\color{codegray},
  stringstyle=\color{codepurple},
  basicstyle=\ttfamily\footnotesize,
  breakatwhitespace=false,
  breaklines=true,
  captionpos=b,
  keepspaces=true,
  numbers=left,
  numbersep=5pt,
  showspaces=false,
  showstringspaces=false,
  showtabs=false,
  tabsize=2
}
\lstset{style=mystyle}

\section{RNN Implementation}
\label{sec:rnn implementation}

We carry out the RNN experiments using PyTorch framework.
In this paper, we restrict ourselves to binary 0/1 rewards.
For all experiments, our policy network is a single layer LSTM followed by LeakyRELU non-linearity and a fully connected layer. 
We use the fixed LSTM latent state dimension of 50, irrespective of numbers of arms.
The implementation of the policy network is provided in the code snippet below:

\begin{lstlisting}[language=Python, caption=Policy Network]
class RecurrentPolicyNet(nn.Module):
  def __init__(self, K=2, d=50):
    super(RecurrentPolicyNet, self).__init__()
    self.action_size = K  # Number of arms
    self.hidden_size = d
    self.input_size = 2*d

    self.arm_emb = nn.Embedding(K, self.hidden_size)     # Number of arms
    self.reward_emb = nn.Embedding(2, self.hidden_size)  # For 0 reward or 1 reward
    self.rnn = nn.LSTMCell(input_size=self.input_size,
                          hidden_size=self.hidden_size)
    self.relu = nn.LeakyReLU()
    self.linear = nn.Linear(self.hidden_size, self.action_size)

    self.hprev = None

  def reset(self):
    self.hprev = None

  def forward(self, action, reward):
    arm = self.arm_emb(action)
    rew = self.rew_emb(reward)

    inp = torch.cat((arm, rew), 1)
    h = self.rnn(inp, self.hprev)
    self.hprev = h

    h = self.relu(h[0])
    y = self.linear(h)

    return y
\end{lstlisting}

To train the policy we use the proposed \gradband algorithm as presented in Alg.~\ref{alg:gradband}. We used a batch-size $m=500$ for all experiments.
Along with theoretically motivated steps, we had to apply a few practical tricks:
\begin{itemize}
    \item Instead of SGD, we used adaptive optimizers like Adam or Yogi \citep{zaheer18adaptive}.
    \item We used an exponential decaying learning rate schedule. We start with a learning rate of 0.001 and decay every step by a factor of 0.999.
    \item We used annealing over the probability to play an arm. This encourages exploration in early phase of training. In particular we used temperature = $1/(1-\exp(-5i/L))$, where $i$ is current training iteration and $L$ is the total number of training iterations.
    \item We applied curriculum learning as described in Section~\ref{sec:rnn experiment}.
\end{itemize}
Our training procedure is highlighted in the code snippet below.

\begin{lstlisting}[language=Python, caption=Training overview]
optimizer = torch.optim.Adam(policy.parameters(), lr=0.001)
scheduler = torch.optim.lr_scheduler.ExponentialLR(optimizer, 0.999)

...

probs = rnn_policy_network(previous_action, previous_reward)  
m = Categorical(probs/temperature)   # probability over K arms with temperature
action = m.sample()                  # select one arm
reward = bandit.play(action)         # receive reward

...

loss = -m.log_prob(action) * (cummulative_reward - baseline)  # Eq (3)
loss.backward()   # Eq (9)
optimizer.step()
scheduler.step()

...
\end{lstlisting}